\documentclass[11pt]{article}
\usepackage[T1]{fontenc}
\usepackage[square,sort,semicolon,numbers]{natbib}
\newcommand{\Pareto}{Pareto }
\newcommand{\mab}{MAB }
\newcommand{\SUBP}{Sub-Pareto }
\newcommand{\Subp}{sub-Pareto }

\newcommand{\alg}{\textsc{BL-Moss}}
\newcommand{\Alg}{\textsc{BL-Moss }}
\newcommand{\ucb}{\textsc{UCB1}}

\newcommand{\moss}{\textsc{Moss}}
\newcommand{\Moss}{\textsc{Moss }}

\newcommand{\Bmab}{BL-MAB}
\newcommand{\bmab} {BL-MAB }

\usepackage{amssymb,amsthm,amsmath,amssymb,wrapfig,dsfont,authblk}
\usepackage{thmtools,thm-restate}
\PassOptionsToPackage{hyphens}{url}
\usepackage{varioref}
\usepackage{verbatim}
\usepackage[usenames,svgnames,xcdraw,table]{xcolor}
\definecolor{DarkBlue}{rgb}{0.1,0.1,0.5}
\definecolor{DarkGreen}{rgb}{0.1,0.5,0.1}
\usepackage[backref=page]{hyperref}
\hypersetup{
     colorlinks   = true,
     linkcolor    = DarkBlue, 
     urlcolor     = DarkBlue, 
	 citecolor    = DarkGreen 
}
\renewcommand*{\backref}[1]{}
\renewcommand*{\backrefalt}[4]{%
    \ifcase #1 (Not cited.)%
    \or        (Cited on page~#2)%
    \else      (Cited on pages~#2)%
    \fi}
\usepackage[margin=0.9in]{geometry}
\usepackage{enumerate}
\usepackage{graphicx}
\usepackage{subdepth}
\usepackage{enumitem}
\usepackage{tcolorbox}
\usepackage[linesnumbered, ruled]{algorithm2e}
\usepackage{mathtools}
\usepackage{caption}
\usepackage{bbm}
\usepackage{soul}
\usepackage{mathrsfs,amsfonts,dsfont,authblk}
\usepackage{thmtools,thm-restate}
\usepackage[capitalise,noabbrev]{cleveref}
\usepackage[font=footnotesize,labelsep=newline,justification=centering]{caption}
\crefname{property}{Property}{Properties}
\usepackage{nicefrac}
\usepackage{multicol}
\usepackage{enumerate}
\usepackage{libertine}
\usepackage{wrapfig}
\usepackage[backgroundcolor=blue!3!white]{todonotes}
\usepackage[labelfont={normalfont,bf},textfont=it]{caption}
\usepackage{tikz}
\usetikzlibrary{positioning}
\usetikzlibrary{calc}
\newtheorem{theorem}{Theorem}

\newtheorem{lemma}{Lemma}

\theoremstyle{definition}

\newtheorem{observation}{Observation}
\newtheorem{claim}{Claim}
\newtheorem{definition}{Definition}

\theoremstyle{remark}

\usepackage{mdframed}

\usepackage{tikz}
\usetikzlibrary{shapes,arrows,decorations.markings}
\usepackage{titlesec}
\titlespacing*{\section}{0pt}{0.5cm}{0.10cm}
\titlespacing{\subsection}{0pt}{2pt}{1pt}
\usepackage{lipsum}

\newcommand\blfootnote[1]{%
  \begingroup
  \renewcommand\thefootnote{}\footnote{#1}%
  \addtocounter{footnote}{-1}%
  \endgroup
}

\title{ Ballooning Multi-Armed Bandits }

\begin{document}
\author{Ganesh Ghalme    \thanks{Technion, Israel  Institute of Technology, Israel. \texttt{ganeshg@campus.technion.ac.il} } \quad Swapnil Dhamal\thanks{Chalmers University of Technology, Sweden.   \texttt{swapnil.dhamal@gmail.com}} \quad Shweta Jain \thanks{Indian Institute of Technology, Ropar, India. \texttt{shwetajains20@gmail.com}}\\  \quad Sujit Gujar \thanks{International Institute of Information
Technology. \texttt{sujit.gujar@iiit.ac.in}}  \quad Y. Narahari \thanks{Indian Institute of Science. \texttt{narahari@iisc.ac.in}}}
\date{ }
\maketitle
\blfootnote{$^{**}$A part of this work was done while the first author was at Indian Institute of Science, India.}

\begin{abstract}
 
In this paper, we introduce {\em ballooning multi-armed bandits\/} (\bmab), a novel extension of the classical stochastic MAB model. In the \bmab\  model, the set of available arms grows (or balloons) over time. In contrast to the classical MAB setting where the regret is computed with respect to the best arm overall, the regret in a \bmab\ setting is computed with respect to the best available arm at each time. We first observe that the existing stochastic MAB algorithms result in linear regret for the \bmab\ model. We prove that, if the best arm is equally likely to arrive at any time instant, a sub-linear regret cannot be achieved. Next, we show that if the best arm is more likely to arrive in the early rounds, one can achieve sub-linear regret.  Our proposed algorithm determines (1) the fraction of the time horizon for which the newly arriving arms should be explored and (2) the sequence of arm pulls in the exploitation phase from among the explored arms. Making reasonable assumptions on the arrival distribution of the best arm in terms of the thinness of the distribution's tail, we prove that the proposed algorithm achieves sub-linear instance-independent regret. We further quantify explicit dependence of regret on the arrival distribution parameters. We reinforce our theoretical findings with extensive simulation results. We conclude by showing that our algorithm would achieve sub-linear regret even if (a) the distributional parameters are not exactly known, but are obtained using a reasonable learning mechanism or (b) the best arm is not more likely to arrive early, but a large fraction of arms is likely to arrive relatively early.




\end{abstract}
\blfootnote{$\dagger \dagger$  An original version of this work is accepted for publication in the  Journal of Artificial Intelligence (AIJ) by Elsevier. A previous, preliminary version  of  this  paper  appeared as an extended abstract in proceedings of the 19th International Conference on Autonomous Agents and Multiagent Systems (2020), Auckland, New Zealand  \cite{ghalme2020ballooning}.}

\section{Introduction}
\label{sec:BL-intro}

The classical stochastic multi-armed bandit (MAB) problem provides an elegant abstraction to a number of important sequential decision making problems. In this setting, the planner chooses (or pulls) from a fixed pool of finitely many actions (i.e., arms), a single arm at each discrete time instant upto arbitrary time horizon. Each arm, when pulled, generates a reward from a fixed but a priori unknown stochastic distribution corresponding to the pulled arm. The planner's goal is to minimize the regret, that is, the loss incurred in the expected cumulative reward due to not knowing the reward distribution of the arms beforehand. The MAB problem encapsulates the classical exploration versus exploitation dilemma, in that the planner's algorithm has to arrive at an optimal trade-off between exploration (pulling relatively unexplored arms) and exploitation (pulling the best arms according to the history of pulls thus far). This  problem has been extensively studied in the literature. These studies include analysis of the achievable lower bound on regret~\cite{lai85},  bandit  algorithms~\cite{auer2010ucb,  thompson1933likelihood, garivier2011kl}, empirical studies~\cite{chapelle2011empirical, devanand2017empirical, russo2018tutorial}, and several extensions to the standard model~\cite{slivkins2019introduction, bubeck2012regret}. Many papers  show that some of the well known algorithms such as UCB1 \cite{auer2002finite} , \textsc{Thomson Sampling} \cite{agrawal2012analysis, kaufmann2012thompson}, \textsc{KL-UCB} \cite{garivier2011kl}  are known to attain asymptotically optimal regret guarantee upto a problem-dependent constant.  The above  list is  far from exhaustive. We refer the reader to \cite{slivkins2019introduction, lattimore2020bandit} for a book exposition on the topic.

The theoretical results in MAB are complemented by a wide variety of modern applications such as internet advertising~\cite{babaioff2009characterizing, nuara2018combinatorial}, crowdsourcing~\cite{JAIN2018}, clinical trials~\cite{villar2015multi}, and wireless communication~\cite{maghsudi2014joint}, which can be modeled in the MAB setup.  Due to a wide range of  applications and an elegant theoretical foundation, several  variants of the MAB problem have been proposed. Contributing to the long line of work that studies different variants of bandits, in this paper, we introduce a novel variant of MAB  which we call {\em Ballooning multi-armed bandits\/} (\bmab). In contrast to the classical MAB where the set of available arms is fixed throughout the run of an algorithm, the set of arms in \bmab\ grows (or balloons) over time.

 To see that the traditional algorithms are not regret-optimal in the \bmab\ setting, consider the following thought experiment. Let a new arm arrive at each time instant in decreasing order of mean rewards and let the MAB algorithm run for a total of $T$ time instants. The traditional MAB algorithms (such as \ucb, \moss, etc.) would pull the newly arrived arm at each time instant, thus incurring a regret of $O(T)$. 
 Note in the above example that even while the best arm appeared at the first time instant itself, traditional algorithms end up pulling all other arms at least once, which leads to a high regret. As the set of available arms expands over time, the traditional algorithms could not sufficiently explore each of the arms to identify the best arm.  Also, note that the regret in \bmab\  depends not only on the mean reward  of the arms, but also on when they arrive. Hence, any \bmab\ algorithm ought to be aware of the arrival of the arms. 
 
 It is clear from the above example that traditional algorithms cannot provide sublinear regret guarantees in \bmab\ setup as there is not enough time spared for  exploitation. Further, as the number of arms increases (potentially linearly) with time, an optimal algorithm must ignore (or drop)  a few arms. Hence, in addition to achieving an optimal trade-off between the number of exploratory pulls and exploitative pulls, the algorithm must also ensure that it does not drop  too many (or too few) arms.
 
\subsection{Motivation}

The \bmab\ framework is directly applicable in any scenario where the set of options grows over time, and, the objective is to choose the best option available at any given time. We motivate the practical significance of \bmab\  with a few applications.

A contemporary example is provided by  question and answer (Q\&A) platforms such as  Reddit, Stack Overflow, Quora, Yahoo! Answers, and ResearchGate, where for a given question,  the platform's goal is to discover the highest quality answer that should be displayed in the most prominent slot.
Each answer post is modeled as a distinct arm of a \bmab\ instance, and the rewards are distributed according to a Bernoulli distribution parameterized  by the quality of the posted answer. Note that this quality is a priori unknown to the platform and hence needs to be learnt. For this, the platform employs certain endorsement mechanisms with indicators such as upvotes, likes, and shares (or re-posts). If a user likes the answer displayed to her, then she may endorse the answer.
Each display of a posted answer corresponds to a pull of the corresponding arm.  At each time instant, a new user observes the existing answer posts shown by the platform, and decides whether to endorse them. Further, the user may also choose to  post her own answer, thus increasing the number of available arms. Hence, the number of available arms (answers)  monotonically increases over time.

The problem of learning qualities of the answers on Q\&A forums has been modeled under the MAB framework in various studies~\cite{ghosh2013learning, tang2019bandit, LIU18}. However, these studies resort to the existing MAB variations which are not well suited for Q\&A forums. For instance, in~\cite{ghosh2013learning},  the problem is modeled with a classical MAB framework by limiting the number of arms via strategic choice of an agent, by assuming that a user incurs a certain cost for posting an answer and hence posts it only if she derives a positive utility by doing so. However, a user's  behavior on the platform may be driven by simple cognitive heuristics rather than a well calibrated strategic decision~\cite{burghardt2016myopia}. 
In another work~\cite{LIU18}, 
the number of arms is limited by randomly dropping some of the arms from consideration. The regret is then computed with respect to only the considered arms. That is, they do not account for the regret incurred due to the randomly dropped arms.  

Some of the other applications of \bmab\ framework are in various websites that feature user reviews, such as Amazon and Flipkart (product reviews), Tripadvisor (hotel reviews), IMDB (movie reviews), and so on. 
As time progresses, the reviews for a product (or a hotel or a movie) keep arriving, and the website aims to display the most useful reviews for that product (or hotel or movie) at the top. The usefulness of a review is estimated using users' endorsements for that review, similar to that in Q\&A forums.
\bmab\ is also applicable in scenarios where users comment on a video or news article, on a video or news hosting website, where the website's objective is to display the most popular or interesting comment at the top.



The \bmab\ setting thus provides a natural framework to be considered in such type of applications. It needs an independent investigation owing to a number of reasons. For instance, 
one of the MAB variants that holds some similarity with \bmab\ is
sleeping multi-armed bandit (S-MAB)~\cite{kleinberg2010,chatterjee2017analysis}, 
where
a subset of a fixed set of base arms is available at each time instant. Though the S-MAB framework captures the availability of a small subset of arms at each time, it assumes that the set of base arms is fixed and is small as  compared to the time horizon.  In contrast, the \bmab\ framework allows for the number of available arms to increase, potentially linearly with time. Hence, an optimal sleeping bandits algorithm such as \textsc{Auer}~\cite{kleinberg2010} would end up incurring a linear regret in  \bmab\ setting.

Another MAB variant with some similarities to BL-MAB is the many-armed (potentially infinite) bandit 
\cite{wang2009algorithms, carpentier2015simple, berry1997bandit},
where 
the number of arms could be potentially equal to or greater than the time horizon.  
Berry et al. \cite{berry1997bandit} consider the case of an infinite arm bandit with  Bernoulli reward distribution. However, they assume that the optimal arm has a quality of $1$, which is seldom the case in practical applications. 
%
Other investigations considering infinitely many arms \cite{wang2009algorithms,carpentier2015simple} make certain assumptions on the distribution of the near optimal arm, in order to achieve sub-linear regret. A further difference is that in the many-armed bandit setting, researchers usually assume the existence of a smooth function relating the mean rewards of the arms (for instance, \cite{carpentier2015simple}). Here, we consider that the reward distributions are independent with arbitrary means. This rules out any side information that could be gathered from the pulls of other arms. Finally, all the above works consider that all the arms are available in all time instants, and hence use the traditional notion of regret. In our case, the regret incurred by an algorithm in a given time instant is the difference between the quality of the best available arm during that time and the quality of the arm pulled by the algorithm (same as the notion of regret considered in sleeping bandits). The \bmab\  framework is thus an interesting blend of both the sleeping bandit model and the many-armed bandit model.  
\subsection{Our Contributions}
Following are the main contributions of this paper:
\begin{itemize}
\item We introduce the \bmab\ model that allows the set of arms to grow over time.
\item For the BL-MAB model, we show that, without any distributional assumptions on the arrival time of the highest quality arm, the regret grows linearly with time (Theorem \ref{thm:impossibility}).
    \item We propose an algorithm (\alg) which determines: (1)~the fraction of the time horizon until which the newly arriving arms should be explored at least once and (2)~the sequence of arm pulls during the exploitation phase. 
     Our key finding is that \Alg achieves sub-linear regret  under practical and minimal assumptions on the arrival distribution of the best arm, namely, sub-exponential tail (Theorem \ref{thm:sublinearTwo}) and \Subp tail (Theorem \ref{thm:sublinearOne}). Note that we make no assumption on the arrival of the other arms. As the regret depends on the qualities of the  arms and the sequence of their arrivals, it is interesting  that with sub-exponential and \Subp assumption on only the best arm's arrival pattern, we can achieve sub-linear regret. 
    \item We carry out  a pertinent simulation study to empirically observe
how the expected regret varies with the time horizon. We  
find a strong validation for our theoretically derived regret bounds.
\item We study the cost of parametric uncertainty, which we define to be the loss incurred due to not knowing the parameters of the best arm's arrival distribution exactly (Theorems~\ref{thm:uncertainty_subexp} and \ref{thm:uncertainty_subpareto}).
We also show that our algorithm is applicable to the setting which does not make  distributional assumptions on the arrival time of the best arm, but instead, on the rate of arrival of arms with time (Theorem~\ref{thm:equivalence}).


\end{itemize}

The paper is organized as follows. In Section~\ref{sec:BL-model}, we present our proposed \bmab\  model. In Section \ref{sec:BL-lower_bound}, we show that no algorithm  can achieve sub-linear regret in the most general setup. Hence, an additional assumption on the arrival of arms is warranted. We define two distributional assumptions on the arrival time of the best arm which would enable us to achieve sub-linear regret. Next, we present some preliminaries  in Section \ref{sec:BL-prelims}, followed  by our  proposed algorithm and its theoretical analysis in Section \ref{sec:BL-results}. Section \ref{sec:BL-simulation} presents our simulation results. We study two extensions in Section \ref{sec:extensions}, namely, relaxing the distributional assumption on the arrival of the best arm, and deducing the cost of parametric uncertainty.
We conclude with related work (Section \ref{sec:BL-related}) and future directions (Section \ref{sec:BL-discussion}).  

\section{The Model}
\label{sec:BL-model}
A classical MAB instance is given by the tuple $ \langle  K, (\mathcal{D}_i)_{i \in K} \rangle $. Here, $K$ is a fixed set of arms and $\mathcal{D}_{i}$ is the reward distribution corresponding to an arm $i$.  Denote by $q_i$, the mean of distribution $\mathcal{D}_i$.  Consider that each of the distributions $\mathcal{D}_{i}$ is supported over a  finite interval and is unknown to the algorithm. Throughout the paper, without loss of generality, we consider that $\mathcal{D}_i$ is supported over $[0,1]$. Further, we will refer to $q_i$ as the quality of arm $i$. A \mab  algorithm is run in discrete time instants, and the total number of time instants is denoted by time horizon $T$. In each time instant  aka round, the algorithm selects a single arm and observes the reward corresponding to the selected arm. The arms which are not selected, do not give any reward. More precisely, a \mab algorithm is a mapping from the history of arm pulls and obtained rewards, to a distribution over the set of arms. 

At each time instant, a \bmab\ algorithm chooses a single arm from the set of available arms and receives a reward generated randomly according to the reward distribution $\mathcal{D}_{i}$ of the chosen arm $i$.  New arms may spring up at each time instant. Throughout the paper, we consider that at most one  new arm arrives at each time, and  the arms are never dropped. 
Let $K(t)$ denote the set of arms available at round $t$. In the \bmab\ model, this set of available arms grows by at most one arm per round, i.e., $K(t) \subseteq K(t+1)$ and $|K(t)| \leq |K(t+1)| \leq |K(t)| + 1  $. A \bmab\ instance, therefore, is given by $\langle T, (K(t),(\mathcal{D}_{i})_{ i \in K(t)})_{t =1}^{T} \rangle$.

Similar to the notion of regret in the sleeping stochastic MAB model~\cite{kleinberg2010}, we introduce the notion of regret in \bmab\ setting that takes into account the availability of the arms at each time $t$. 
Let $i_t$ denote the arm pulled by the algorithm and $i_t^{\star}$ be the best available arm at time $t$, i.e., $i_t^{\star} = \arg \max_{i \in K(t)} q_i$. Further, let $\mathcal{I}$ denote a \bmab\ instance and $A$ be a \bmab\ algorithm. 
 The distribution-dependent regret of $A$ is given by $$\mathcal{R}_{A}(T, \mathcal{I})  = 
    \mathbb{E}\big[ \sum_{t=1}^{T}(q_{i_t^{\star}} - q_{i_t}) \big].$$ Throughout the paper, we consider distribution-free regret given as $\mathcal{R}_{A}(T)  = 
     \sup_{\mathcal{I}}  \mathcal{R}_{A}(T,\mathcal{I})$.  Note that the distribution-free regret bound is a worst case regret bound over all the arrival sequences of the arms and all possible reward distributions. In the next section, we show that  for the \bmab\ setting, it is not possible to achieve sublinear distribution-free regret bound.

\section{Lower Bound on Regret}
\label{sec:BL-lower_bound}

As pointed out in Section~\ref{sec:BL-intro},
it is clear that UCB-style algorithms (which pull arms based on uncertainty) would pull each incoming arm at least once, leaving no rounds for exploitation. Hence, they incur linear regret in the ballooning bandit setup\footnote{In particular, when $|K(t)| =t$.}. However, it is not obvious that a different, more sophisticated algorithm (such as the one which randomly drops some arms) would not be able to achieve sub-linear regret. Our first result (Theorem~\ref{thm:impossibility}) shows that no algorithm can attain sub-linear regret without any distributional assumption on the best arm's arrival.

\begin{theorem}
\label{thm:impossibility}
There exists a BL-MAB instance $\mathcal{J}$ such that any MAB algorithm \textsc{Alg}\ satisfies $$\mathcal{R}_{\textsc{Alg}}(T, \mathcal{J}) = \Omega(T)$$
\label{thm:lowerBound}
\end{theorem}
\begin{proof}
We prove the theorem in three steps. In the first step, we construct a BL-MAB instance $\mathcal{J}$ such that $\mathcal{J}$ has a single best arm and all the suboptimal  arms  have the same quality parameter. Further, we consider that each arm is equally likely to be the best arm, i.e., the probability that an arriving arm is the best arm is $\frac{1}{T}$ for all $t=\{ 1,2, \ldots, T\}$. Next, in Step 2, we simulate any  \textsc{Bl-MAB} algorithm $\textsc{Alg}$ by simulation algorithm $\textsc{Sim}$ such that $\mathcal{R}_{\textsc{Alg}}(T, \mathcal{J}) = \mathcal{R}_{\textsc{Sim}}(T, \mathcal{J}) $.  Finally, in Step 3, we show  $\mathcal{R}_{\textsc{Sim}}(T,\mathcal{J}) = O(T) $ for every simulated algorithm $\textsc{Sim}$.
We begin with the construction of a BL-MAB instance $\mathcal{J}$.
\newline

\noindent
\textbf{Step 1:} A new arm arrives at each discrete time instant $t$, till the predetermined time horizon $T$. There is a single best arm $i^{\star}$ with quality parameter $q_{i^{\star}} = 1/2 + \varepsilon$.  Here, $\varepsilon > 0$ is a problem-independent constant.  Each suboptimal arm $i \neq i^{\star}$ has $q_i = 1/2$. Further,  each arm is equally likely to be the best arm, i.e., $\mathbbm{P}(i = i^{\star}) = \frac{1}{T}$ for all $i \in [T]$. Next, we show an important property of the BL-MAB instance $\mathcal{J}$ (Claim \ref{clm:prop1}).  In particular, we show that for any algorithm, there exists a corresponding arm-pulling strategy which  pulls the arms in the order of their arrival and has the same expected reward.
\newline

\noindent
\textbf{Step 2:} Consider a single run of any BL-MAB algorithm $\textsc{Alg}$ on instance $\mathcal{J}$ and  let $G$ denote the set of distinct arms pulled till time $T$, i.e., $G = \{ i\in[T] |  N_{i,T}^{\textsc{Alg}}>0\} $, with $g = |G|$. Here, $N_{i,t}^{\textsc{Alg}}$ is the number of times arm $i$ is pulled till (and excluding) time instant $t$ by $\textsc{Alg}$. We drop the superscript when the algorithm is clear from the context.    Further, let $M_n = \{1,2,\ldots, n\}$ be the collection of the first\footnote{In the order of their arrival.} $n$ arms. We simulate  $\textsc{Alg}$ on $\mathcal{J}$ as using a simulation $\textsc{Sim}$ such that $\mathcal{R}_{\textsc{Alg}}(T,\mathcal{J}) = \mathcal{R}_{\textsc{Sim}}(T,\mathcal{J})$. For any arm pull $i_t$ at time $t$, we pull arm $i_t^{'}$ in the simulation $\textsc{Sim}$ as follows.

\begin{center}
\boxed{    \!\begin{aligned}
  &\textsc{Sim}\\
  &i_t^{'} = \begin{cases} i_t & \text{if } i_t \in M_{g} \\
   \min \{ i \in  M_{g} | N_{i,t}^{\textsc{Sim}} = 0  \} & \text{if $i_t \in  G \setminus  M_{g}$ and $N_{i_t,t}^{\textsc{Alg}} =0$} \\
   i_{\ell}^{'}  & \text{if $ i_t \in G \setminus M_{g}$ and $N_{i_t,t}^{\textsc{Alg}} > 0$} \newline \\ (  \ell = \min\{m < t: N_{i_t , m+1}^{\textsc{Sim}} =1\} )  \\
\end{cases}
  \end{aligned}
}
\end{center}
Whenever \textsc{Alg} pulls an arm $i_t \in G \setminus M_{g}$ for the first time, $\textsc{Sim}$ assigns a corresponding $i_t^{'} \in M_g$. Let us say that $\textsc{Alg}$ pulls 3 distinct arms in its run (i.e., $G = \{1,4,6\}$) and the sequence of arms pulled is given by  $(1,1,1,4,4,6,6,4,1,6,\ldots ,1)$. In this case, $\textsc{Sim}$ will pull arms 1, 2 and 3 in sequence, $(1,1,1,2,2,$ $3,3,2,1,3, \ldots ,1)$. That is, all the arm pulls of arms from set  $\{1,2,3\}$  are retained and all the arms from outside this set that are pulled are replaced by the arms from this set as follows: the least index arm is assigned to the first arm encountered from outside the set, i.e., whenever $\textsc{Alg}$ pulls arm 4, $\textsc{Sim}$ pulls arm 2 (similarly, whenever $\textsc{Alg}$ pulls arm 6, $\textsc{Sim}$ pulls arm 3).
We now prove that both $\textsc{Alg}$ and $\textsc{Sim}$ have the same expected rewards.
\begin{claim}
$\mathcal{R}_{\textsc{Alg}}(T,\mathcal{J}) = \mathcal{R}_{\textsc{Sim}}(T,\mathcal{J}) $.
\label{clm:prop1}
\end{claim}
\begin{proof}

\begin{align*}
    \mathcal{R}_{\textsc{Alg}} (T, \mathcal{J}) & = \mathbbm{E}\big[ \sum_{t =1}^{T} X_{i_t}^{\star} - X_{i_t}\big] = \sum_{t=1}^{T} \mathbbm{E_{\textsc{Alg}}}\big [ q_{i_t^{\star}} - q_{i_t} \big] \\
    & = \sum_{t=1}^{T} \mathbbm{E_{\textsc{Alg}}}\Delta(i_t^{\star}, i_t)
\end{align*}

If $i_t^{'} = i_t$, we immediately have that $\Delta_{i_t^{\star}, i_t} =  \Delta_{i_t^{\star}, i_t^{'}}$. Hence, without loss of generality, let  $ i_t^{'} \neq i_t  $. We have

\begin{align*}
\small
    \Delta_{i_t^{\star}, i_t} &= \big [ (1/2 + \varepsilon) \cdot \mathbb{P}(i^* \text{ has arrived before } t) + 1/2 \cdot \mathbb{P}(i^* \text{ arrive after time }t ) \big ]  \\ & \hspace{40pt} - \big [ (1/2 + \varepsilon) \cdot \mathbbm{P}(i_t = i^{\star})   +  1/2 \cdot \mathbb{P}(i_t \ne i^*)  \big ] \\
     &= \big [ (1/2 + \varepsilon) \cdot \big ( \sum_{\ell =1}^{t}\mathbbm{P}(i_{\ell} = i^{\star}) \big ) + 1/2 \cdot \big ( 1-  \sum_{\ell =1}^{t}\mathbbm{P}(i_{\ell} = i^{\star}) \big ) \big ]  \\ & \hspace{40pt} - \big [ (1/2 + \varepsilon) \cdot \mathbbm{P}(i_t = i^{\star})  +  1/2 \cdot  (1 - \mathbbm{P}(i_t = i^{\star}))  \big ] \\
    &  = \big [ (1/2 + \varepsilon) \cdot \big ( \sum_{\ell =1}^{t}\mathbbm{P}(i_{\ell}^{'} = i^{\star}) \big ) + 1/2 \cdot \big ( 1-  \sum_{\ell =1}^{t}\mathbbm{P}(i_{\ell}^{'} = i^{\star}) \big ) \big ]  \\ & \hspace{40pt} - \big [ (1/2 + \varepsilon) \cdot \mathbbm{P}(i_t^{'} = i^{\star})  +  1/2 \cdot  (1 - \mathbbm{P}(i_t^{'} = i^{\star}))  \big ] \tag{As $\mathbbm{P}(i_{\ell} = i^{\star}) = \mathbbm{P}((i_{\ell}^{'} = i^{\star}))=1/T $ for all $\ell \in [T]$}\\
    & = \Delta_{i_t^{\star}, i_t^{'}}
\end{align*}
This completes the proof.
\end{proof}
Henceforth, we will focus only on the simulation algorithms.
\\

 \noindent
\textbf{Step 3:} Let $G$ denote the set of arms pulled by the algorithm \textsc{Sim} in its run. The regret of $\textsc{Sim}$ can be written as
$$\mathcal{R}_{\textsc{Sim}}(T) = \mathbbm{E}_{G} \Big [ \sum\limits_{i=1}^{|G|} \mathbbm{P}(i^{\star} \in G , i \neq i^{\star} ) \mathbbm{E}[N_{i,T}] \cdot \varepsilon + \sum \limits_{i = |G| +1}^{T} \mathbbm{P}(i = i^{\star}) (T-i) \cdot \varepsilon \Big ]. $$
The outer expectation is with respect to the number of arms pulled by any (possibly randomized) algorithm \textsc{Sim}. For a fixed value of $G$, the inner expectation represents the number of times arms $i \in G$ are pulled till the time horizon $T$. Using a classical result from \cite{lai85}, we have  $\mathbbm{E}[N_{i,T}] \geq \eta \log(T)$ for some $\eta > 0$ that depends on the suboptimality of the arm $i$. Using this, we have
\begin{align*}
  \mathcal{R}_{\textsc{Sim}}(T)  & \geq  \mathbbm{E} \Big[ \sum\limits_{i=1}^{|G|} \mathbbm{P}( i \neq i^{\star} | i^{\star} \in G ) \cdot \mathbbm{P}(i^{\star}  \in G ) \eta \log(T)   + \sum \limits_{i = |G| +1}^{T} \mathbbm{P}(i = i^{\star}) (T-i)  \Big]  \cdot \varepsilon \\
   & = \mathbbm{E} \Big[ \sum\limits_{i=1}^{|G|} \frac{|G|-1}{|G|}\cdot \frac{|G|}{T}  \eta \log(T)   + \sum \limits_{i = |G| +1}^{T} \frac{T-i}{T}   \Big]  \cdot \varepsilon \\
   & = \mathbbm{E} \Big [ \sum |G|(|G|-1) \cdot  \eta \log(T)   + \frac{(T-|G| -1)(T -|G|)}{2} \Big] \cdot \frac{\varepsilon}{T} \\
   & = \mathbbm{E} \Big [ (1 + 2 \eta \log(T)) |G|^2 - (2(T - \eta \log(T)) -1) |G| + T^2 - T \Big ] \cdot \frac{\varepsilon}{2T} \\
   & \geq \min_{|G| \in [0,T]} \Big [ (1 + 2 \eta \log(T)) |G|^2 - (2(T - \eta \log(T)) -1) |G| + T^2 - T \Big ] \cdot \frac{\varepsilon}{2T}.
\end{align*}
Note that the above expression is quadratic in $|G|$. For $T \leq 1/2 + \eta \log(T) $, the minimum occurs when the value of $|G|$ is  1. In this case, the regret is $\Omega(T)$. For  $T  > 1/2 + \eta \log(T)$, the minimum occurs when $|G| = \frac{2(T - \eta \log(T)) -1}{2(1 + 2 \eta \log(T))}$. For this case, we have
\begin{align*}
    \mathcal{R}_{A}(T)     &\geq \Big[\frac{(2(T-\eta\log(T)) -1)^{2}}{4(1+ 2\eta\log(T))} - \frac{(2(T-\eta\log(T)) -1)^{2}}{2(1+ 2\eta\log(T))}   + T^{2} -T \Big]\cdot \frac{\varepsilon}{2T} \\
    &= \Big[  T^{2} -T - \frac{(T - \eta\log(T) - 1/2)^{2}}{(1+ 2\eta\log(T))}  \Big] \cdot \frac{\varepsilon}{2T} \\
    & >  \Big[\frac{ (T - 1/2)}{2}  \frac{2\eta\log(T)}{1+2\eta\log(T)} -1/4\Big] \cdot \varepsilon \\
    & = \Omega(T).
\end{align*}
\end{proof}

Theorem \ref{thm:impossibility} provides a strong impossibility result on the achievable distribution-free regret bound under \bmab\  setting. However, one can still achieve sub-linear regret by imposing appropriate structure on the  \bmab\ instances. Observe that the regret depends on the arrival of arms, i.e., $(K(t))_{t=1}^T$, and their reward distributions $(\mathcal{D}_{i})_{i\in K(t)}$. We impose  restrictions on the arrival of the best arm $i^{\star} = \arg\max_{i \in K(T)} q_i$ so that the probability that $i^{\star}$ arrives early is large enough; this would allow a learning algorithm to  explore  the  best arm enough to estimate the true quality of that arm with high probability. As noted previously, the other arms may arrive arbitrarily. Further, note that we make no assumption on the qualities of individual arms.
\subsection{{Arrival of the Best Arm}}
Let $X$ be the random variable denoting the time at which the best arm arrives. Further, let $F_{X}(t)$ denote the cumulative distribution function of $X$. In our first result, we use the following sub-exponential tail assumption  on the arrival time of the best arm.
\\

\noindent \textbf{Sub-exponential tail:} \emph{
There exists a constant $\lambda > 0$ such that the probability of the best arm arriving later than $t$ rounds, is upper  bounded by  $e^{-\lambda t}  $, i.e., $  F_{X}(t) > 1 - e^{-\lambda t} $. }
\\

Next, we consider a relaxed condition on the tail probabilities,
i.e., when the tail does not shrink as fast as in the sub-exponential case. We consider the family of distributions whose tail is thinner than that of  \Pareto distribution.
\\

%


\noindent
\textbf{\SUBP tail:} \emph{There exists a constant $\beta > 0$ such that the probability of the best arm arriving later than $t$ rounds, is upper  bounded by  $t^{-\beta}  $, i.e., $  F_X(t) > 1 - t^{-\beta} $. }
\\


The aforementioned assumptions naturally arise in the context of Q\&A forums as observed in extensive empirical studies on the nature of answering as well as voting behavior of the users. Anderson et al.~\cite{anderson2012discovering} observe that high reputation users hasten to post their answers early. One possible explanation for this phenomenon could be that the  users who are motivated by the visibility that their answers receive, tend to be more active on the platform and also provide high quality answers early on, which explains their reputation scores.
Thus, it is reasonable to assume that the best answer arrives, with high probability, in early rounds.

Note that the uniform distribution is the limiting case of the sub-exponential case, when $\lambda=0$. We will show that, while the uniform distribution results in linear regret (Theorem \ref{thm:impossibility}), a sub-linear regret can be achieved for \bmab\ instances having the best arm arrival distribution with even slightly thinner tail than that of uniform distribution (Section~\ref{sec:BL-results}).


\section{Preliminaries}
\label{sec:BL-prelims}
    We now present some  essential concepts which will be useful for our analysis in the remainder of the paper.
\subsection{Lambert $W$ Function}
\begin{definition}
    \label{def:lambert}
       For any $x>-e^{-1}$, the Lambert $W$ function, $W(x)$, is defined as the solution to the equation $w e^{w} = x$, i.e., $W(x)e^{W(x)} = x$.
    \end{definition}
It is easy to check that in the non-negative domain,
Lambert $W$ function satisfies the following regularity properties~\cite{hoorfar2008inequalities}. The detailed proofs are provided in  Appendix~\ref{app:prop}.
\begin{restatable}[]{property}{PropZero}
   \label{prop:zero}
   The Lambert $W$ function can be equivalently written as the inverse of the function $f(x) := xe^{x}$, i.e., $W(xe^{x}) =x$.  
\end{restatable}

\begin{restatable}[]{property}{PropOne}
   \label{prop:one}
   For any $x \geq e$, we have  $\log(x)/2 < W(x) \leq  \log(x) $.
   \end{restatable}
   \begin{restatable}[]{property}{PropTwo}
   \label{prop:two}
   For any $x \in [0, \infty)$, the Lambert $W$ function is unique, non-negative, and strictly increasing.
   \end{restatable}

It can be noted that it is easy to numerically approximate $W(x)$ for a given $x$, using Newton-Raphson's  or Halley's method. Moreover, there exist efficient numerical methods for evaluating it to arbitrary precision \cite{Corless1996}.


   \subsection{MAB Algorithms}

 We now review some of the MAB algorithms, starting with UCB1, perhaps the most famous stochastic MAB algorithm. We then review \textsc{Thompson Sampling} \cite{thompson1933likelihood} which presents an alternate Bayesian approach to the MAB problem.
 In this paper, we use  \Moss (Minimax Optimal Strategy in the Stochastic case)~\cite{audibert2010regret} as an underlying learning algorithm; the \Moss algorithm uses UCB-style indexing of the arms.
 In principle, one could use any underlying learning algorithm in a BL-MAB setup. However, as we shall discuss, one needs to carefully tune the thresholding parameter for the learning algorithm in question.
 \newline

 \noindent
 \textbf{\textsc{UCB1}}
\newline
UCB1, proposed in \cite{auer2002finite}, is perhaps the most famous stochastic MAB algorithm. At each time $t$, UCB1 maintains a UCB index for each arm and pulls an arm with the highest UCB index. In particular, UCB1 pulls an arm $i_t$ such that
$$i_t \in  \arg \max_{i \in K} \Bigg [\hat{q}_{i, N_{i,t}} +  \sqrt{\frac{2\log(t)}{N_{i,t}}}\Bigg ].  $$
    Here, $K = \{1,2, \ldots, k \}$ denotes the set of arms and  $N_{i,t}$ is the number of times arm $i$ was pulled before (and excluding) round $t$ and $\hat{q}_{i,N_{i,t}}$ are the empirical estimates of the arm $i$  from $N_{i,t}$ samples.
\newline

\noindent
 \textbf{\textsc{Thomson Sampling}}
 \newline
 First proposed in \cite{thompson1933likelihood}, the theoretical regret guarantee of \textsc{Thompson Sampling} remained an open problem for over 80 years before  \cite{agrawal2012analysis} and \cite{kaufmann2012thompson} independently showed that \textsc{Thompson Sampling} achieves asymptotically optimal regret guarantee (upto problem-dependent constant). This Bayesian approach maintains a conjugate prior distribution for each  arm. We refer the reader to \cite{agrawal2012analysis} for the detailed algorithm as well as a regret analysis of \textsc{Thompson Sampling}.
\newline

\noindent
\textbf{\textsc{Moss}} \newline
For a fixed number of $k$ arms, the   \Moss algorithm pulls an arm $i_t$ at time $t$ where $$i_t \in  \arg \max_{i \in K} \Bigg [\hat{q}_{i, N_{i,t}} + \sqrt{\frac{\max(\log(\frac{T}{k \cdot N_{i,t}}), 0)}{N_{i,t}}}\Bigg ].$$
    %
    %
    Each arm is pulled once in the beginning, and ties are broken arbitrarily.
\\

In contrast to other  popular MAB algorithms such as \textsc{Thompson Sampling}~\cite{thompson1933likelihood}, \textsc{UCB1} \cite{auer2002finite} and \textsc{KL-UCB}  \cite{garivier2011kl},  \textsc{Moss} simultaneously achieves the optimal instance-dependent as well as  optimal instance-independent regret guarantee \cite{audibert2010regret}. However, the time horizon is assumed to be known to the algorithm a priori. The problem of achieving simultaneous optimal any-time regret guarantees had remained open until recently, when modified versions of KL-UCB algorithms, namely, \textsc{KL-UCB++}  \cite{menard2017minimax} and \textsc{KL-UCB-Switch} \cite{garivier2018klucbswitch},   were proven to be simultaneously optimal. However, the instance-independent regret bound of these algorithms still depends linearly on the number of available arms (Theorem 1 in \cite{menard2017minimax} and Theorem 4 in \cite{garivier2018klucbswitch}, respectively).

As we shall see in Algorithm \ref{alg:B-MOSS} in the next section, we need a threshold parameter $\alpha$ which signifies the fraction of arms that a learning algorithm should explore.
This parameter  must be tuned based on the regret guarantees of the learning algorithm, i.e., the internal regret in the BL-MAB framework and the external regret. We choose \textsc{Moss} for its simplicity and optimality (both in terms of number of arms and the time horizon).
For instance, \textsc{Moss} achieves an optimal instance-independent regret guarantee of $O(\sqrt{kT})$. Other algorithms such as \textsc{Thompson Sampling}, UCB1 or KL-UCB may also be employed as underlying learning algorithms, albeit with a slight ($O(\sqrt{\log(T)})$) increase in  internal regret.
   We leave the determination of the threshold parameter and the corresponding regret analysis of BL-MAB using other algorithms as an interesting direction for future work.



\section{  The BL-MOSS  Algorithm and Regret Analysis }
\label{sec:BL-results}
\subsection{The \Alg Algorithm}
We now present our algorithm,  \Alg (Algorithm \ref{alg:B-MOSS}), that   uses \Moss  as the underlying learning algorithm. The number of arms explored by \Alg is dependent on the distribution of  arrival of the best arm. In particular, \Alg considers only the first $\lceil \alpha T\rceil $ arms in its execution ($\alpha \in (0,1]$). Later in this section, we show how to derive the value of $\alpha$ for distributions with sub-exponential and \Subp tails.
 Observe that the proposed \Alg algorithm is a simple extension of \Moss  and  is  practically easy to implement. Further, \Moss does not assume any structure on the arrival of suboptimal arms. Thus, we are able to obtain sub-linear regret with minimal assumptions.



\begin{algorithm}[t!]
\KwIn{Time horizon $T$, Distributional parameter $\lambda >0$ or $\beta>0$}
\vspace{-1mm}
\[
\hspace{-1.5mm}
\text{Set } \alpha :=
\begin{cases}
\frac{W(2\lambda T)}{2\lambda T} \;\text{ under sub-exponential tail property}\\
T^{\frac{-2\beta}{1+2\beta} } \;\;\;\text{ under \Subp tail property}
\end{cases}
\]
  \For{  $t = 1,2, \ldots  T$  }{
  \KwIn{A newly arriving arm at time $t$}
  \eIf{ $|K(t)| \leq \lceil \alpha  T \rceil$}{\moss($K(t)$)}{\moss($K(\lceil \alpha  T \rceil)$)
  }
 }
\caption{\Alg}
  \label{alg:B-MOSS}
\end{algorithm}
\subsection{Regret Analysis of \Alg}
   We begin with an upper bound on the expected regret of the \Moss algorithm. Note that \Moss achieves  optimal (up to a constant factor) regret bound. Throughout the paper, we use the notation \moss(k) to denote that the \Moss algorithm is run with $k$ arms.
   \begin{restatable}{theorem}{mossRegret}\cite{audibert2010regret}
   \label{thm:moss_regret}
   For any time horizon $T\geq 1$, the expected regret of \Moss is given by $  \mathcal{R}_{\moss(k)}(T) \leq 6 \sqrt{k T }$.
   \end{restatable}

For a given \bmab\ instance $\mathcal{I}$, let $j^{\star} \!=\! \arg \max_{i \in K(\lceil \alpha T \rceil) } q_{i} $ and $i^{\star} \!=\! \arg \max_{i \in K(T)} q_i$.
Clearly, we have that $q_{i^{\star}} \geq q_{j^{\star}}$.  As stated earlier, the regret of the algorithm can be decomposed into internal regret, i.e., the regret incurred by the learning algorithm considering only $ \lceil \alpha T \rceil $ arms and external regret, i.e., the regret incurred by \Alg due to the fact that \Alg  might have ignored the best arm. Write  $\Delta(i,j) =  q_i - q_j  $ and let $t_i$ be the time of arrival of arm $i$. Further, let $i_t^{\star}$ denote the best arm till time $t$. The distribution-dependent regret $\mathcal{R}_{\Alg} (   T , \mathcal{I})$ is given as

\begin{equation}
        \mathbb{P}(i^{\star} = j^{\star})  \Big [\underbrace{  \sum_{t=1}^{t_{j^{\star}} -1}\Delta(i_t^{\star}, i_t) +  \sum_{t= t_{j^{\star}}}^{T }\Delta(j^{\star}, i_t)  }_{\mathcal{R}_{\Alg}^{\text{int}}(  T)} \Big ]
       + \mathbb{P}(i^{\star } \neq j^{\star}) \underbrace{ \Big[  \sum_{t= 1}^{t_{i^{\star}}-1 } \Delta(i_t^{\star}, i_{t}) + \sum_{t= t_{i^{\star}}}^{T }\Delta({i^{\star}, i_t})\Big ]   }_{\mathcal{R}_{\Alg}^{\text{ext}}(T)}
       \label{eqn:int_ext_regret}
\end{equation}

The first and the second terms respectively denote the  internal regret and the external regret of \alg. We ignore the ceiling in $\lceil \alpha T \rceil$ throughout this section to avoid notation clutter.

Note that $\mathcal{R}_{\moss(L)}(T) \leq \mathcal{R}_{\moss(K)}(T) $ for all $L \subseteq K$ and for any any time horizon $T$.
From Theorem \ref{thm:moss_regret}, we have the following observation about the internal regret of \alg.
\begin{observation}
\label{obs:reg_moss}
For the value of $\alpha$ computed by \alg, we have $\mathcal{R}_{\Alg}^{\text{int}}( T ) \leq \mathcal{R}_{\moss(\alpha T)}(    T) \leq 6 \sqrt{\alpha  } T $.
\end{observation}

The first inequality in Observation \ref{obs:reg_moss} follows from the fact that  in a classical \textsc{MAB} setting, all the arms are available at all times, whereas in a $\textsc{BL-MAB}$ setting, arms arrive online. Hence, the best arm is available at all times in MAB, whereas in BL-MAB, the arrival of the best arm is delayed.

To bound the overall regret, we begin with the following lemma which  explicitly shows the relation between the  expected  regret of the algorithm  and  $F_{X}(\cdot)$. Recall that the random variable $X$ denotes the time of arrival of the best arm.
\label{sec:results}
    \begin{lemma}
    \label{lem:regret}
        The upper bound on  the expected regret for any \bmab\ instance is given by
    $\mathcal{R}_{\alg}(T) \leq T ( 1 - (1 - 6 \cdot \sqrt{\alpha})  F_{X}(\alpha T)) $, with \Alg  exploring only the first $\alpha T$ arrived arms.
    \end{lemma}
\begin{proof}
For a given \bmab\ instance $\mathcal{I}$, let  $t_i$ denote the time at which  arm $i$ becomes available for the first time. Let $i^{\star}$ denote the best arm till $T$ rounds, i.e., $i^{\star} = \arg\max_{i \in K(T)}q_i $. Further, let $j^{\star}$ be the best arm among the arms considered by \alg, i.e., $j^{\star} = \arg\max_{j \in K(\alpha T )} q_i$.  Notice that $K(\alpha T) \subseteq K(T)$. This implies $q_{i^{\star}} \geq q_{j^{\star}}$.
    \allowdisplaybreaks
\begin{align*}
    \hspace{-2mm}
    \mathcal{R}_{ \alg}( T, \mathcal{I})
        & \leq \mathbb{E}\big[ \sum_{t=1}^{\alpha T}(q_{j^{\star}} - q_{i_t})  +  \sum_{t= \alpha T + 1}^{T}(q_{i^{\star}} - q_{i_t}) \big ] \tag{$\because q_{i^{\star}} > q_{j^{\star}}$}
        \\[-5pt]
    & = \mathbb{P}(i^{\star} = j^{\star} ) \Big[ \sum_{t=1}^{T}(q_{j^{\star}} - q_{i_t}) \big ]  \\ & +  \mathbb{P}(i^{\star} \neq  j^{\star} ) \Big [ \sum_{t=1}^{\alpha T}(q_{j^{\star}} - q_{i_t}) +  \sum_{t=\alpha T + 1}^{ T}(q_{i^{\star}} - q_{i_t}) \big ]   \\
& \leq  6\mathbb{P}(i^{\star} = j^{\star} )  \sqrt{\alpha T \cdot T} +  \sum_{t=1}^{ T}(q_{i^{\star}} - q_{i_t})\mathbb{P}(i^{\star} \neq j^{\star})  \tag{From Observation \ref{obs:reg_moss} and since $q_{i^{\star}} \geq q_{j^{\star}}$ }\\
    & \leq  6   T  \sqrt{\alpha}  \cdot  \mathbb{P}(i^{\star} = j^{\star})     +  \mathbb{P}(i^{\star} \neq  j^{\star}) T  \tag{$ \because \sum_{t=1}^{ T}(q_{i^{\star}} - q_{i_t})    \leq T$ }\\
    & =  6   T  \sqrt{\alpha}  \cdot  \mathbb{P}(t_{i^{\star}} \leq \alpha T)    +  (1-\mathbb{P}(t_{i^{\star}} \leq \alpha \cdot T)) T \\
    & = T(1 - (1 - 6 \cdot \sqrt{\alpha}) \mathbb{P}(t_{i^{\star}} \leq \alpha T))  \\
        & = T(1 - (1- 6 \cdot \sqrt{\alpha}) F_{X}(\alpha  T))
\end{align*}
Note that the above inequality holds for any  \bmab\  instance and hence we have $\mathcal{R}_{ \alg}( T) =  \sup_{\mathcal{I}} \mathcal{R}_{ \alg}( T, \mathcal{I})  \leq T(1 - (1- 6 \cdot \sqrt{\alpha}) F_{X}(\alpha  T)) $
\end{proof}

\subsubsection{Sub-exponential tail distribution}

We now show that under the sub-exponential tail property on $X$, \Alg achieves sub-linear regret. We begin with the following  lemma that lower bounds the probability of the arrival of the best quality arm in the initial $\alpha T$ rounds.
\begin{lemma}
 \label{lem:lemmaTwo}
 Let the arm arrival distribution of the best arm satisfy sub-exponential tail property  for some $\lambda \geq  0$. Then for any $c>0$ and $\alpha \geq  \frac{W(\lambda T/c)}{ \lambda T/c}$, we have that  $F_X(\alpha T ) > (1 - \alpha^{c}) $.
 \end{lemma}
\begin{proof}
Note that from the Property \ref{prop:one} of the  Lambert W function we have  $ \frac{\log(x)}{2}\leq W(x) \leq \log(x) $ for $x \geq e$. We have,
\begin{align*}
\alpha & \geq \frac{ W(\lambda T/c)}{\lambda T/c} \implies \frac{ \alpha \lambda T }{c}\geq W(\lambda  T /c)
  \implies    W \big ( \frac{\alpha \lambda T}{c} \cdot e^{ \alpha \lambda T/c} \big ) \geq W(\lambda T/c)  \tag{ by Property \textbf{\ref{prop:zero}}} \\
 \implies & \frac{\alpha \lambda T}{c}  \cdot e^{ \alpha \lambda T /c} \geq \lambda T /c 
\end{align*}
So, we have $1 - \alpha^{c} \leq 1 - e^{ - \lambda ( \alpha  T)} < F_X(\alpha T)$. The last inequality follows from the sub-exponential tail property.
\end{proof}

    \begin{restatable}{theorem}{forExponential}
\label{thm:sublinearTwo}
 Let the arrival distribution of the best arm  satisfy the  sub-exponential tail property for some $\lambda \geq 0$, and let $T$ be large enough such that $T>\frac{36 c \log(36)}{\lambda}$ for some $c>0$. Then with $\alpha = \frac{W(\lambda T /c)}{\lambda T /c}$, the upper bound on the expected regret of \alg,   $\mathcal{R}_{\Alg}(T)$, is   \mbox{$ O \Big(T \cdot \max\big( e^{-c W(\lambda T /c)},  e^{- \frac{W(\lambda T /c)}{2} } \big ) \Big)$}. The upper bound on the expected regret is minimized when  $c = 1/2$ and is given  by $O \Big(\sqrt{\frac{T \log(2\lambda T)}{2\lambda}}\Big)$.
 \end{restatable}


\begin{proof}
From Lemma \ref{lem:lemmaTwo}, we have  $ F_X(\alpha T) > 1 - \alpha^c $ for all $\alpha \geq \frac{W(\lambda T /c)}{\lambda T /c}$. 
Thus, from Lemma \ref{lem:regret}, we have $\mathcal{R}_{\alg}(T) < T(1- (1- 6 \cdot \sqrt{\alpha}) (1 - \alpha^{c}))$.

\vspace{2mm}
Note that for achieving sub-linear regret, it is necessary that $(1- 6\cdot\sqrt{\alpha}) $ is strictly  positive, for which it is necessary that  $\alpha < 1/36 $. From Lemma \ref{lem:lemmaTwo}, we also have  $\alpha \geq \frac{W(\lambda T /c)}{\lambda T/c}$.  Since such a feasible $\alpha$ may not exist for small values of $T$, we consider that $T$ is large enough.
It can be easily shown that $\frac{W(\lambda T /c)}{\lambda T/c} < 1/36 \Longleftrightarrow T > \frac{36 c \log(36)}{\lambda } \approx \frac{129 c}{\lambda}$
(see Claim~\ref{clm:Tbound} in Appendix~\ref{sec:BL-appendix}).
%

\vspace{2mm}
 Thus,  for $ 1/36 > \alpha \geq \frac{W(\lambda T /c)}{\lambda T /c}$, we have:   $\mathcal{R}_{\alg}(T) < T( 6 \cdot \sqrt{\alpha} + \alpha^c - 6\cdot\alpha^{c + 1/2} )$.
   Recall that by definition, we have $\alpha \leq 1 $. Thus when $c \in (0,1/2]$, the term $\alpha^{c}$ dominates the other terms in the regret expression, whereas when $c > 1/2$, the term $\sqrt{\alpha}$   dominates. We analyze these cases separately.

\vspace{2mm}
\noindent
\textbf{Case 1 ($c\in(0, 1/2]$): } In this case, the regret is given by $\mathcal{R}_{\Alg}(T) =  O(\alpha^{c} T )$. Note that the regret is minimized for the lowest feasible value of $\alpha$, i.e., $\alpha = \frac{W(\lambda T /c)}{\lambda T /c}$, resulting in $\mathcal{R}_{\Alg}(T) = O\Big(T \big  ( \frac{W(\lambda T /c)}{\lambda T /c}\big)^{c}\Big) = O(T \cdot e^{-c W(\lambda T /c)})$. The last equality follows from the equivalent definition of Lambert W function (Property \ref{prop:zero}).

\vspace{2mm}
\noindent
\textbf{Case 2 ($c\in[1/2,\infty )$): }
 In this case, the regret is given by $\mathcal{R}_{\Alg}(T) =  O(\sqrt{\alpha} T )$. Again, the regret is minimized when  $\alpha =  \frac{W(\lambda T /c)}{\lambda T /c}$. The regret in this case is given by $\mathcal{R}_{\Alg}(T) =  O\Big( T\cdot\sqrt{\frac{W(\lambda T /c)}{\lambda T/c}}\Big) = O \big(T \cdot e^{\frac{-W(\lambda T/c )}{2}})$.

\vspace{1mm}
 Further, we have that    in   Case 1, $ e^{-c W(\lambda T /c)} >  e^{\frac{- W(2 \lambda T )}{2} } $ for any $c\in(0,1/2)$ (see Claim \ref{clm:decreasing} in Appendix~\ref{sec:BL-appendix}). For  Case 2, we have from Property  \ref{prop:two} that, $W(\lambda T /c)$ is decreasing in $c$, which gives us that $e^{\frac{- W(2 \lambda T )}{2}} < e^{\frac{- W( \lambda T/c )}{2}}$
 for any $c \in  (1/2, \infty)$. This  shows that the minimum regret is achieved when $c = 1/2$, and the regret is given by
 $\mathcal{R}_{\Alg}(T) =  O\Big( \sqrt{\frac{T \cdot W(2\lambda T )}{2\lambda }}\Big) =  O\Big( \sqrt{\frac{T  \log(2\lambda T )}{2\lambda }}\Big)$.
 The last inequality follows from Property \ref{prop:one}, since $2 \lambda T \geq e$ $(\because T > \frac{36 c \log(36)}{\lambda} \text{ where } c=1/2)$.
 %
\end{proof}

If we absorb $\lambda$ (which is a constant with respect to $T$) in order notation, we have
    $R_{\Alg} = O(\sqrt{T \log(T)})$.


\subsubsection{\SUBP tail distribution}

We now prove the sub-linear regret of \Alg
under the \Subp tail property.
\begin{restatable}{lemma}{lemmaOne}
\label{lem:lemmaOne}
Let the arm arrival distribution of the best arm satisfy \Subp tail property for some $\beta>0$.  Then for any $c>0$ and $\alpha \geq  T^{\frac{-\beta}{c+\beta}}$, we have that  $F_X(\alpha T ) > (1 - \alpha^{c}) $.
\end{restatable}

\begin{proof}
First note that $ \alpha \geq T^{\frac{- \beta}{c + \beta}} \iff \alpha^c \geq (\alpha T)^{-\beta}$. This implies that $(1 - \alpha^c) \leq 1 - (\alpha T)^{-\beta}$. Further, from the \Subp tail property, we have that  $1 - (\alpha T)^{-\beta} < F_X(\alpha T)$.
\end{proof}

\begin{restatable}{theorem}{forSublinear}
\label{thm:sublinearOne}
Let the arrival distribution of arms satisfy the  \Subp tail property for some $\beta>0$, and let $T$ be large enough such that $T > (36)^{\frac{c+\beta}{\beta}}$ for some $c>0$. Then with $\alpha =   T^{\frac{-\beta}{\beta+c}}$, the upper bound on the expected regret of \mbox{\alg},  $\mathcal{R}_{\Alg}(T) $,  is $O(\max(T^{\frac{c + \beta(1-c)}{c+\beta}}, T^{\frac{2c+\beta}{2(c+\beta)}}))$. The upper bound on the expected regret is minimized when  $c = 1/2$ and is given  by $O(T^{\frac{1+\beta}{1+2\beta}})$.
    \end{restatable}



\begin{proof}
From Lemmas \ref{lem:regret} and \ref{lem:lemmaOne}, we have $\mathcal{R}_{\alg}(T) < T(1- (1-  6 \cdot \sqrt{\alpha}) (1 - \alpha^c))$.
For achieving sub-linear regret, it is necessary that $(1- 6\cdot\sqrt{\alpha}) $ is strictly  positive. So, we should have  $\alpha < 1/36 $.
Further, from Lemma \ref{lem:lemmaOne}, we have  $\alpha \geq T^{\frac{-\beta}{c+\beta}}$.
So, for a feasible $\alpha$ to exist, it is necessary that $T^{\frac{-\beta}{c+\beta}} < 1/36 \Longleftrightarrow T > (36)^{\frac{c+\beta}{\beta}}$, i.e., $T$ is large enough.
  Thus, for $1/36 > \alpha \geq T^{\frac{-\beta}{c +   \beta}}$, we have
    $\mathcal{R}_{\alg}(T) <  T( 6 \cdot  \sqrt{\alpha} + \alpha^c - 6 \cdot  \alpha^{c + 1/2 })$.
    As earlier, we analyze  two cases.
    %

\vspace{2mm}
\noindent
\textbf{Case 1 ($c\in(0, 1/2]$): } In this case, the regret is given by $\mathcal{R}_{\Alg}(T) =  O(\alpha^{c} T )$. The minimum regret is obtained when $\alpha = T^{\frac{-\beta}{\beta + c}}$ and is given by $  O( T^{1 -\frac{c\beta}{c+\beta}})$.

\vspace{2mm}
\noindent
\textbf{Case 2 ($c\in[1/2,\infty )$): }
 In this case, the regret is given by $\mathcal{R}_{\Alg}(T) =  O(\sqrt{\alpha} T )$. Again, the regret is minimum when $\alpha = T^{\frac{-\beta}{\beta + c}}$ and is given by $  O( T^{\frac{2c +\beta}{2(c+\beta)}})$.

\vspace{1mm}
 Furthermore, it is easy to see that  in Case 1, $T^{\frac{1+\beta}{1+2\beta}} < T^{\frac{\beta + c(1-\beta)}{c+\beta}} $ for any $c\in(0,1/2)$. Similarly, in Case 2,
 $T^{\frac{1+\beta}{1+2\beta}} <   T^{\frac{2c + \beta }{2(c+\beta)}}$ for any $c \in (1/2, \infty)$. This  shows that the minimum regret is achieved when $c = 1/2$.
\end{proof}

\subsection{Important Observations}
We conclude the section with some key observations and remarks.
\begin{observation}
\label{obs:lambdalimits}
If the best arm arrival satisfies sub-exponential tail property with parameter $\lambda$, then
\begin{enumerate}
    \item $\mathcal{R}_{\textsc{BL-Moss}}(T) \rightarrow 0$ as $\lambda \rightarrow \infty$
    \item $\mathcal{R}_{\textsc{BL-Moss}}(T) \rightarrow O(T)$ as $\lambda \rightarrow 0$
\end{enumerate}
\end{observation}
\begin{proof}
Recall that, from Equation~(\ref{eqn:int_ext_regret}), we have
\begin{align*}
&  \mathcal{R}_{\textsc{BL-Moss}}(T) = \\
        & \mathbb{P}(i^{\star} = j^{\star})  \Big [\underbrace{  \sum_{t=1}^{t_{j^{\star}} -1}\Delta(i_t^{\star}, i_t) +  \sum_{t= t_{j^{\star}}}^{T }\Delta(j^{\star}, i_t)  }_{\mathcal{R}_{\Alg}^{\text{int}}(  T)} \Big ]
       + \mathbb{P}(i^{\star } \neq j^{\star}) \underbrace{ \Big[  \sum_{t= 1}^{t_{i^{\star}}-1 } \Delta(i_t^{\star}, i_{t}) + \sum_{t= t_{i^{\star}}}^{T }\Delta({i^{\star}, i_t})\Big ]   }_{\mathcal{R}_{\Alg}^{\text{ext}}(T)}
\end{align*}
Note that, for large enough  $\lambda $ such that $ \frac{W(2\lambda T)}{2\lambda T} \leq \frac{1}{T}$, we have that  $\lceil \alpha T \rceil = 1$ and the algorithm pulls arm $1$ at all times i.e. $i_t = 1$ for all $t\leq T $. Furthermore, as $\lceil \alpha T \rceil = 1$, we have $j^{\star} =1$. Hence,  the internal regret  is zero.   The total regret of the algorithm is, hence,
\begin{align*}
    \mathcal{R}_{\textsc{BL-Moss}}(T) &=  \mathcal{R}_{\textsc{BL-Moss}}^{\text{ext}}(T)  \\ &=  \mathbb{P}(i^{\star } \neq 1)\Big[  \sum_{t= 1}^{t_{i^{\star}}-1 } \Delta(i_t^{\star}, 1 ) + \sum_{t= t_{i^{\star}}}^{T }\Delta({i^{\star}, 1})\Big ]  \\  & \leq   \mathbb{P}(i^{\star } \neq 1)   \sum_{t =1}^{T}  \Delta(i^{\star}, 1 ) \tag{$\because \Delta(i_t^{\star}, 1) \leq \Delta(i^{\star},1)$} \\  & \leq T (1 - F_{X}(1)) \\ & \leq Te^{-\lambda} \approx 0 \tag{$\because \lambda \rightarrow \infty$}
\end{align*}
The last inequality follows from the sub-exponential tail assumption.

If $\lambda \rightarrow 0$, we have that $F_X(t)  > 1 - e^{-\lambda t } \rightarrow 0$. Hence, the arrival distribution of the best arm captures the uniform distribution, i.e., $F_X(t)= \frac{1}{T}$ as well. Hence, from Theorem 1, we have that a regret of $O(T)$  is unavoidable.
\end{proof}

\begin{observation}
If the best arm arrival satisfies sub-Pareto tail property with parameter $\beta$, then
\begin{enumerate}
    \item $\mathcal{R}_{\textsc{BL-Moss}}(T) \rightarrow O(\sqrt{T})$ as $\beta \rightarrow \infty$
    \item $\mathcal{R}_{\textsc{BL-Moss}}(T) \rightarrow O(T)$ as $\beta \rightarrow 0$
\end{enumerate}
\end{observation}
\begin{proof}
Note that  under sub-Pareto tail assumption, we have  $1 - F_{X}(2) < 2^{-\beta}$. Hence, as $\beta \rightarrow \infty$, we have that $ F_{X}(2) \rightarrow 1$. The internal regret of Moss is upper bounded by $6 \sqrt{2T}$, whereas the external regret is given as
$$\mathcal{R}_{\textsc{BL-Moss}}^{\text{ext}}(T) \leq \sum_{t=3}^{T}\mathbbm{P}(i_t = i^{\star} ) (T-t) \leq (T-3) (1 - (1 - 2^{-\beta })) \approx 0 .$$ Hence, the total regret is $O(\sqrt{T})$.

The proof of the second part follows on similar line as the second part of  Observation \ref{obs:lambdalimits}.
\end{proof}

\section{Simulation Study}
\label{sec:BL-simulation}
So far, we focused on deriving upper bounds on regret for distributions (on the arrival time of the best arm) having sub-exponential and \Subp tail with different values of $\lambda$ and $\beta$, respectively.
In particular, for the case of \Subp tail, we deduced that the extent of sublinearity of the regret (the exponent of $T$ in the order of the regret) depends on the value of $\beta$. On the other hand, the upper bound on regret for the case of sub-exponential tail had the same order with respect to $T$ for any reasonable value of $\lambda$, albeit with different multiplicative and additive terms for different values of $\lambda$.
In this section, we aim to illustrate how the expected regret varies with the time horizon $T$, and how the empirical exponents compare with their theoretical bounds for different values of $\beta$ and $\lambda$, for time horizons up to $10^6$ rounds.

\subsection{Simulation Setup}

Note that in a traditional MAB setup, a simulation for a larger time horizon $T''$ could be conducted as an extension of a simulation for a smaller time horizon $T'<T''$. In other words, after obtaining the results for time horizon $T'$, the results for time horizon $T''$ can be obtained by running simulations for an additional $T''-T'$  rounds.
However, in the \bmab\ setup where new arms continue arriving with time and the desired time horizon is known, we have seen that the optimal value of $\alpha$ and hence $\lceil \alpha T \rceil$ depends on the time horizon.
Owing to different values of $\lceil \alpha T \rceil$ for different time horizons $T$, the simulation for a time horizon $T'$ are not extendable to time horizon $T''>T'$.
So even if we have simulation results for time horizon $T'$, it is necessary to run a fresh set of simulations for obtaining  results for time horizon $T''>T'$.
In our simulation study, we consider the following values of time horizon:
 $\{1,2,5,7\}\times 10^4, \{1,2,5,7\}\times 10^5, 10^6$.

We consider that a new arm arrives in each round, and the probability that the arm arriving at time $t$ is the best arm is determined by the distribution function $F_X(t)$.
Thereafter, this best arm ($i^{\star}$) is assigned a  quality ($q_{i^{\star}}$) between 0 and 1 uniformly at random, and the rest of the arms are assigned quality parameters between 0 and $q_{i^{\star}}$ uniformly at random.
%
Given a time horizon $T$, the value of $\alpha$ and hence $\lceil \alpha T \rceil$ are obtained based on our theoretical analysis.
The arm to be pulled in a round is determined by Algorithm~\ref{alg:B-MOSS}, wherein the pulled arm generates unit reward with probability equal to its quality, and no reward otherwise (i.e., as per Bernoulli distribution).
The regret in each round is computed as the difference between the quality of the best arm available in that round and the quality of the pulled arm. The overall regret is the sum of the regrets over all rounds from $1$ till $T$.
Note that we are concerned with the regret irrespective of the numerical values of the arms' qualities. So, for a given instance of the arrival of the best arm, we consider the worst-case regret over 50 sub-instances, where the quality parameters assigned to the arms in different sub-instances are independent of each other.
Also, since different instances would have the best arm arriving in different rounds, the expected regret is obtained by simulating over 1000 such random instances and averaging over the corresponding worst-case regret values.

Our primary objective is to observe how the expected regret varies with the time horizon $T$. In order to observe the influence of various sub-exponential and \Subp tail distributions over the arrival time of the best arm, we conduct simulations for different values of parameters $\lambda$ and $\beta$: $\{0.10,0.25,0.50,0.75,1,2,10\}$.
The other objective is to determine the empirical exponent of the plots (i.e., the value of $\gamma$ such that the expected regret is approximately a constant multiple of $T^\gamma$).
To achieve this, we first estimate the constant factor $\xi$ by dividing the expected regret for $T=10^6$ by $T^\gamma$, for a given value of $\gamma$.
We then compute the squared error when attempting to fit the expected regret with $\xi T^\gamma$. Considering candidate values of $\gamma$ to be between 0 and 1 with intervals of 0.01, we deduce the empirical exponent to be the value of $\gamma$ which results in the least squared error.
We also consider another method for determining the empirical exponent:
we produce the line of best fit for the scatter plot of $\log(T)$ versus the log of the expected regret for that $T$;
the slope of this line gives the empirical exponent.
The empirical exponents obtained using the two methods are almost identical (differing by less than 0.01).

\begin{figure}[t!]
	\centering
	\includegraphics[scale=0.6]{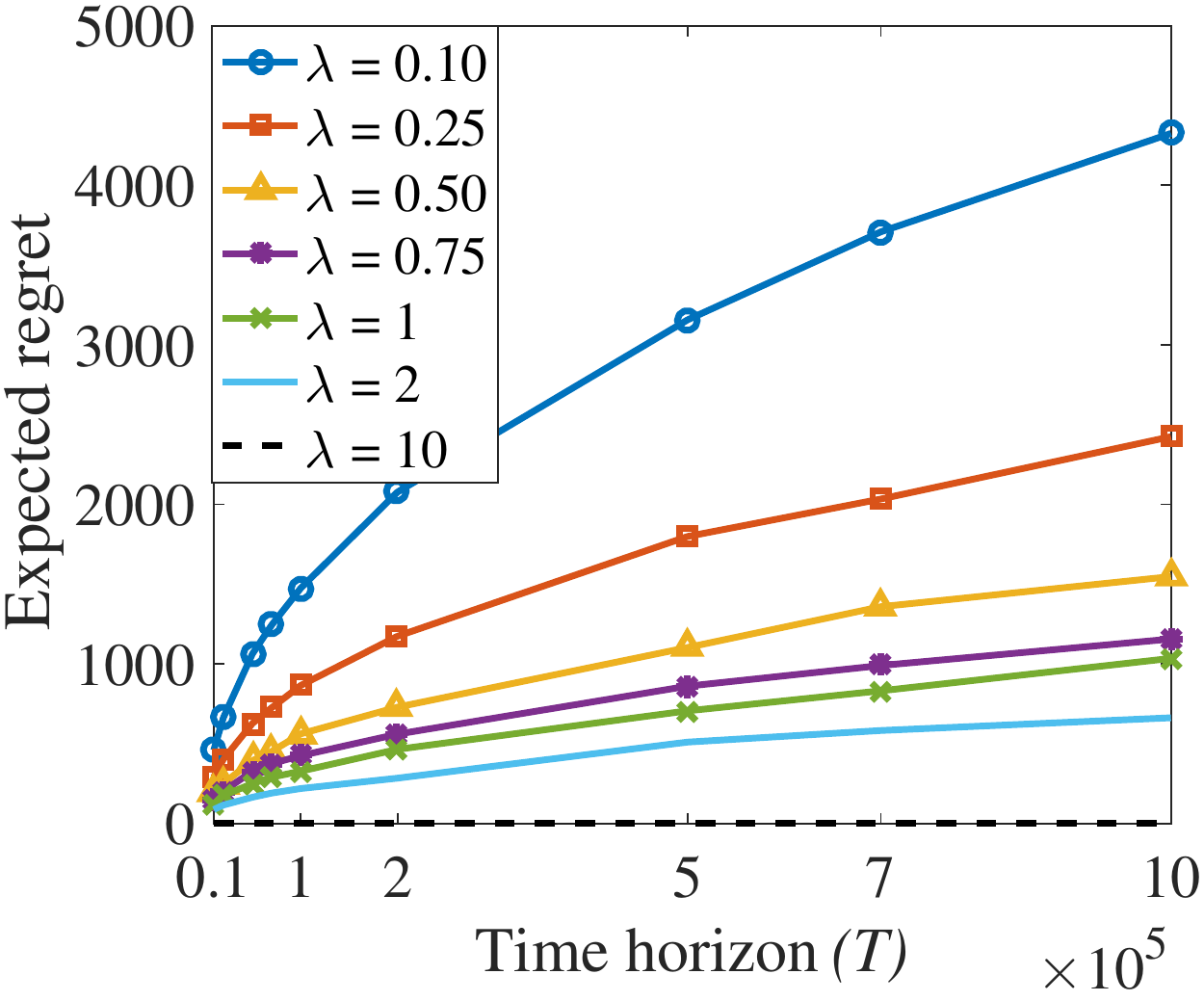}
	\caption{ Expected regret as a function of time horizon for various sub-exponential tails}
	\label{fig:BLplotOne}
	\end{figure}

\subsection{Simulation Results}

As mentioned at the end of our theoretical analysis,
for the sub-exponential tail case when $\lambda \rightarrow \infty$, the upper bound on the expected regret goes to 0. In our simulations with the maximum observed time horizon of $10^6$, the expected regret was observed to be uniformly zero, even for $\lambda=10$ (see Figure~\ref{fig:BLplotOne}).
Further, for other considered values of $\lambda$, the plots exhibit a prominent sub-linear nature.
In particular, considering the maximum observed time horizon of $10^6$,
the empirical exponents for different values of $\lambda$ were consistently observed to be between 0.45 and 0.5
(Theorem \ref{thm:sublinearTwo} showed the order of the regret with respect to $T$, for reasonable values of $\lambda$, to be bounded by $\sqrt{T \log (T)}$, which is an exponent close to 0.5).

\begin{figure}[t!]
	\centering
	\includegraphics[scale=0.6]{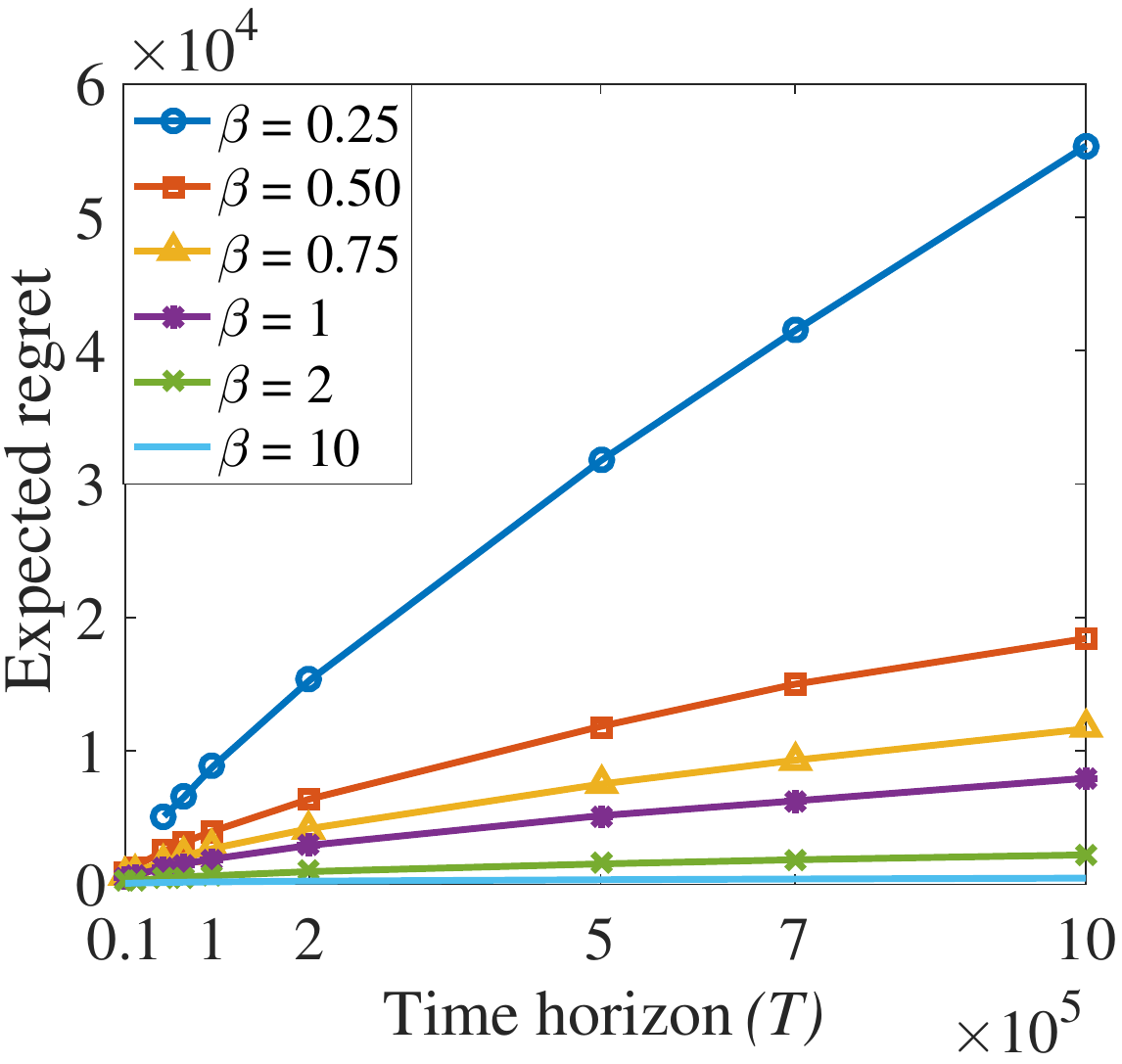}
	\caption{Expected regret as a function of time horizon for various \Subp tails}
	\label{fig:BLplotTwo}
\end{figure}

\begin{figure}[t!]
	\centering
	\includegraphics[scale=0.6]{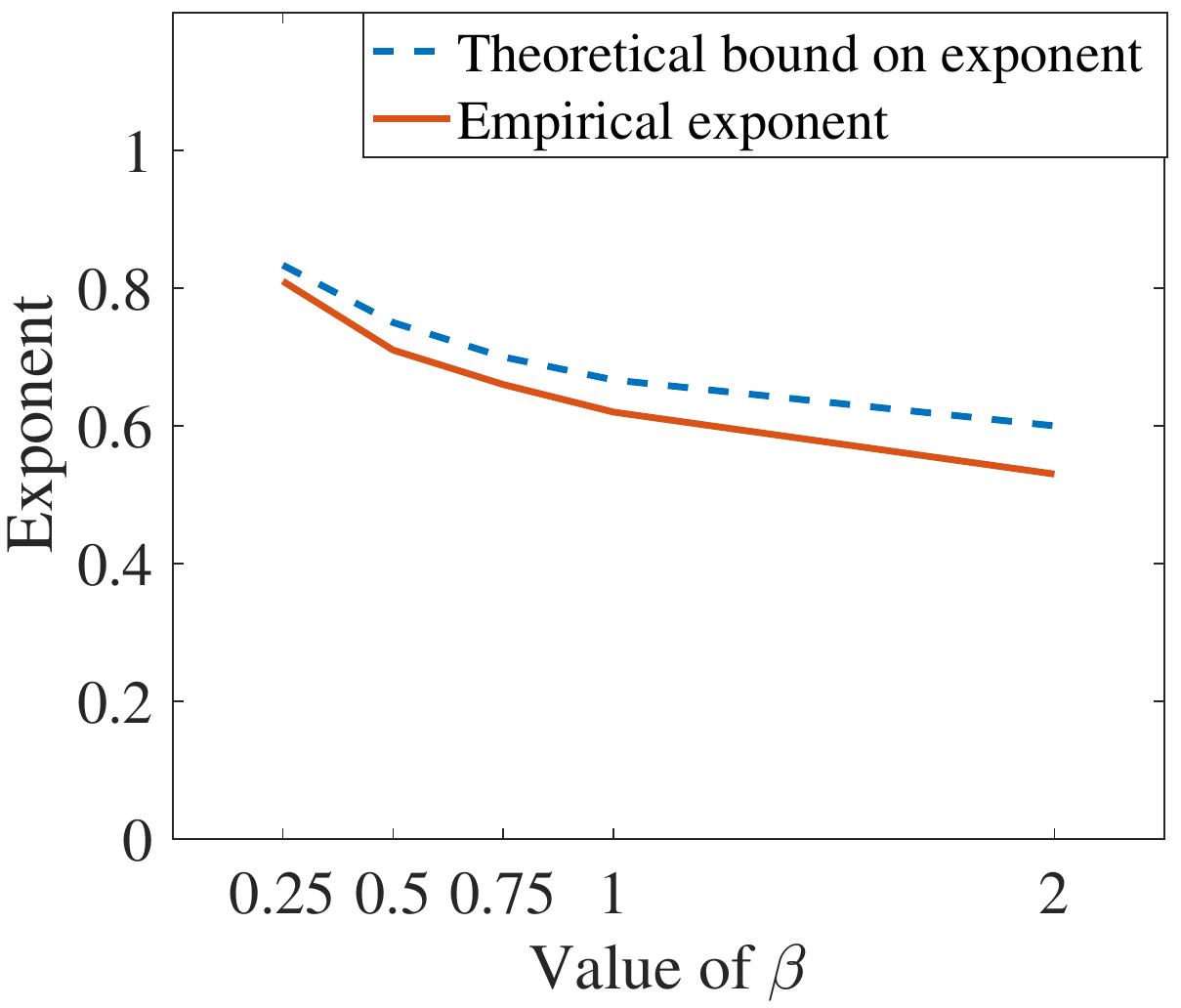}
	\caption{Empirical exponents vs. theoretical bounds for time horizons up to $10^6$}
	\label{fig:BLplotThree}
\end{figure}

For the \Subp tail case illustrated in Figure~\ref{fig:BLplotTwo}, note that we have
no result for $\beta=0.10$ because the value of $T$ for obtaining a feasible $\alpha$ should be greater than $36^6$, which is beyond our maximum observed time horizon of $10^6$.
Moreover, we have
partial results for $\beta=0.25$ because the value of $T$ for obtaining a feasible $\alpha$ should be greater than $36^3$; so the plot starts with $T = 0.5 \times 10^5$.
It can be seen, in general, that the plots in Figure~\ref{fig:BLplotTwo} follow a far less sub-linear nature and exhibit a much higher expected regret than those in Figure~\ref{fig:BLplotOne}.
This is intuitive from our analysis that the sub-exponential tail case is likely to result in a much lower regret than the \Subp tail case.
In particular, the empirical exponent for $\beta=0.25$ was deduced to be 0.8, which is  close to linear (its theoretical upper bound as per our analysis is 0.83).
In general,
considering the maximum observed time horizon of $10^6$,
it can be seen from Figure \ref{fig:BLplotThree} that
the upper bound on the theoretical exponent (which is $\frac{1+\beta}{1+2\beta}$ from Theorem \ref{thm:sublinearOne}) and the empirical exponent are close to each other.


Note that the gap between the empirical exponents and the corresponding theoretical upper bounds could be attributed to the fact that it is  difficult to find the worst-case distribution over the reward parameters of the arms.  Hence, it is unlikely that the    worst-case (or distribution-free) expected regret could be attained in the simulations with a random reward structure.
Since the gap is not very significant, the simulation results suggest that the bounds derived in our regret analysis of \Alg (in Section~\ref{sec:results}) are, in all probability, tight.

\subsection{Additional Notes on Simulation}

It is to be noted that our theoretical analysis holds for any arbitrary time horizon as long as the time horizon is known to \alg. In our simulations, we considered time horizons up to $10^6$ for computational reasons.
The regret is averaged over 1000 random  instances with  same arrival distribution of the best arm. In practice, as only one instance is realized,  the computational overhead is not an impediment in the real world applicability of the proposed algorithm. 

Note also that the standard MAB algorithms (e.g., the UCB family) which are oblivious to the structure on the arrival of arms, would incur linear regret. Also, since these algorithms explore each incoming arm at least once, they would incur linear regret even with sub-exponential or sub-Pareto assumption, when the number of arms grows linearly with time. Our simulations  aimed to observe the order of sublinearity of regret (exponent of $T$). Since existing algorithms would give linear regret, the exponent of $T$ is trivially 1.

\section{Extensions }
\label{sec:extensions}
In this section, we discuss possible extensions and relaxations of the BL-MAB setting studied in the paper. Thus far, we assumed that the true parameter of the arrival distribution of the best arm (i.e., $\lambda $ or $\beta$) is known a priori to the \Alg algorithm. We relax this assumption and consider that the parameter ($\lambda$ or $\beta$) is known only approximately correctly. We show that the proposed algorithm achieves  sublinear regret guarantees even with this relaxation.  Next, we relax the assumption that the best arm arrives early, and instead consider that a large fraction of arms arrive early. We show that our algorithm is applicable even with this alternative assumption;
we validate this assumption with real-world datasets.


 \subsection{Unknown Distributional Parameters of Best Arm's Arrival}
\label{ssec:unknownparam}

So far, we have assumed that the distributional parameters ($\beta$ and $\lambda$) are known to the algorithm designer. In most practical settings, these parameters are not known but can be learnt using  previous data. In this section, we show the effect on the regret bound if learned parameters are used instead of the true parameters.

\subsubsection{Sub-exponential tail distribution}

Let $(X_i)_{i =1}^{n}$ be a collection of i.i.d. random variables sampled from an exponential distribution with parameter $\lambda$ truncated at  $T$. Further, let $T$ be large enough such that $F_{X}(T) = 1 - e^{-\lambda T}  \approx 1$. Let $X = \frac{1}{n}\sum_{i=1}^n X_i$ denote the empirical average  over $n$ questions posted. Let $\hat{\lambda}$ be the estimated parameter of the $n$ i.i.d.  exponential random variables; then $X \approx  \frac{1}{\hat{\lambda}}$. Further, let there be two parameters $\mu$ and $\delta$ such that the Hoeffding's inequality~\cite{hoeffding1956} gives us the following:
\begin{equation}
\label{eq:prob1}
\mathbb{P}\left(\left|\frac{1}{\hat{\lambda}} - \frac{1}{\lambda}\right| \ge \mu \right) \le \delta
\end{equation}


Here, $\mu$ and $\delta$ denote how close the learned parameter and the true parameter are. The value of these parameters depend on the number of samples that we select. For example, in order to achieve confidence of $1-\delta$, we need $\delta \le 2 e^{-\frac{2 n \mu^2}{T^2}}
\implies n \ge \frac{T^2\ln{\left(\frac{2}{\delta}\right)}}{2 \mu^2}$, where $n$ is the number of samples used to learn  parameter $\hat{\lambda}$.

 In Algorithm \ref{alg:B-MOSS}, we will now choose $\alpha = \frac{W(2\hat{\lambda}T)}{2\hat{\lambda}T}$ instead of $\frac{W(2{\lambda}T)}{2{\lambda}T}$ (which is the regret optimal $\alpha$, from Theorem~\ref{thm:sublinearTwo}). Choosing this $\alpha$ would not change Lemma \ref{lem:regret} and we would still have $R_{\Alg}(T) \le T(1-(1-6\sqrt{\alpha})F_X(\alpha T))$. We now have the following theorem on regret.
\begin{theorem}
If \Alg (Algorithm \ref{alg:B-MOSS}) is run with a learned parameter $\hat{\lambda}$ that satisfies Inequality~(\ref{eq:prob1}), then with probability at least $1-\delta$, the regret under  sub-exponential tail distribution assumption is upper bounded by $O\left(T^{\frac{1+\mu\lambda}{2}}\left(\sqrt{\frac{W(2\hat{\lambda}T)}{2\hat{\lambda}}}\right)^{1-\mu\lambda}\right)$ if $\mu\lambda < 1$ and $O\Bigg( \sqrt{\frac{T \cdot W(2\hat{\lambda}T)}{2\hat{\lambda}}}\Bigg)$ otherwise.
The regret is sub-linear in both the cases.
\label{thm:uncertainty_subexp}
\end{theorem}
\begin{proof}
Recall that $\alpha$ is the fraction of arms explored by  \alg.
The following is true with   probability at least $1-\delta$.
\begin{align*}
\allowdisplaybreaks
    &\alpha \geq  \frac{W(\hat{\lambda} T/c)}{\hat{\lambda} T/c} = e^{-W(2\hat{\lambda} T)} \tag*{ ($\because \, c = 1/2$ is regret optimal (Theorem \ref{thm:sublinearTwo}))} \\
   \implies & \log \alpha \ge -W \left(2T\hat{\lambda} \right)\\
    \implies &  W \left(\log\left(\frac{1}{\alpha}\right) e^{\log (\frac{1}{\alpha})} \right) \le W(2T\hat{\lambda}) \tag*{(Property \ref{prop:zero} of Lambert W function)}\\
    \implies & \log\left(\frac{1}{\alpha}\right) \le 2\alpha T\hat{\lambda} \tag*{(Property \ref{prop:two} of Lambert W function)}\\
    \implies & \alpha T \ge \frac{\log(1/\alpha)}{2\hat{\lambda}}\\
    \implies & F_X(\alpha T) > 1- e^{-\frac{\lambda\log(1/\alpha)}{2\hat{\lambda}}} = 1- \alpha^{\frac{\lambda}{2\hat{\lambda}}} \tag*{(sub-exponential tail assumption)}
    \end{align*}
Thus, by Lemma  \ref{lem:regret}, we have
\begin{align*}
\allowdisplaybreaks
R_{\Alg}(T) &\le T\left(1-(1-6\sqrt{\alpha})F_X(\alpha T)\right) \\ & < T\left(1-(1-6\sqrt{\alpha})(1-\alpha^{\frac{\lambda}{2\hat{\lambda}}} )\right) \\
& = T\left(6\sqrt{\alpha} + \alpha^{\frac{\lambda}{2\hat{\lambda}}}  - 6 \alpha^{\frac{1}{2}+\frac{\lambda}{2\hat{\lambda}}} \right) \\ &< T\left(6\sqrt{\alpha} + \alpha^{\frac{\lambda}{2\hat{\lambda}}}  \right)
\end{align*}

Since $\alpha < 1$ and $\frac{1}{\hat{\lambda}} \ge \frac{1}{\lambda} - \mu$, we have
\[
  R_{\Alg}(T) =
  \begin{cases}
    O(T\sqrt{\alpha}), & \text{if } \mu \lambda \ge 1 \\
    O(T\alpha^{\frac{1-\mu\lambda}{2}}), & \text{if } \mu\lambda < 1 \\
  \end{cases}
\]
\textbf{Case 1 $(\mu\lambda < 1)$}:
\begin{align*}
    R_{\Alg}(T) &= O(T\alpha^{\frac{1-\mu\lambda}{2}})\\
    &=O\left(T\left(\frac{W(2\hat{\lambda}T)}{2\hat{\lambda}T}\right)^{\frac{1-\mu\lambda}{2}}\right)\\
    &=O\left(T^{\frac{1+\mu\lambda}{2}}\left(\sqrt{\frac{W(2\hat{\lambda}T)}{2\hat{\lambda}}}\right)^{1-\mu\lambda}\right)
\end{align*}
Since $\mu\lambda < 1$ in this case, the above expression is sub-linear in $T$. If we absorb the constant parameters ($\hat{\lambda}, \lambda,$ and $\mu$), we have $R_{\Alg}(T) = O\left(T^{\frac{1+\mu\lambda}{2}}\left(\sqrt{\log T}\right)^{1-\mu\lambda}\right) \le O\left(T^{\frac{1+\mu\lambda}{2}}\left(\sqrt{\log T}\right)\right)$ (since $\mu > 0$). Note also that if the learned mean parameter $\frac{1}{\hat{\lambda}} $ is close to the true mean parameter $\frac{1}{{\lambda}}$ (i.e., $\mu$ is  close to zero),
 we recover the original  regret guarantee which is $O(\sqrt{T\log(T)})$, with probability at least $1-\delta$.
\\


\noindent
\textbf{Case 2} ($\mu\lambda\geq 1 $):
Here, $\alpha^{\frac{\lambda}{2\hat{\lambda}}} \leq \sqrt{\alpha} \;(\because \alpha \leq 1)$, and hence,
\begin{align*}
\allowdisplaybreaks
R_{\Alg}(T) &\leq O(T\sqrt{\alpha} )
\\
&=
O\Bigg(T \sqrt{\frac{W(2\hat{\lambda}T)}{2\hat{\lambda}T}}\Bigg)
\\
&=
O\Bigg( \sqrt{\frac{T \cdot W(2\hat{\lambda}T)}{2\hat{\lambda}}}\Bigg)
\end{align*}

If we absorb the constant parameter $\hat{\lambda}$ in order notation, we have
    $R_{\Alg}(T) = O(\sqrt{T \log(T)}$, that is, we recover the original  regret guarantee with probability at least $1-\delta$.







\end{proof}

\subsubsection{Sub-Pareto tail distribution}
Let $\hat{\beta}$ be the learned parameter of the sub-Pareto tail distribution.  Like in the sub-exponential case, let there be two parameters $\mu$ and $\delta$ such that
\begin{equation}
\label{eq:prob2}
\mathbb{P}\left(\left|\frac{\hat{\beta}}{\hat{\beta}-1} - \frac{\beta}{\beta-1}\right| \ge \mu \right) \le \delta
\end{equation}
Note that for \Subp tail distribution, mean is defined only for $\beta>1$, and thus we assume that $\beta > 1$ in the rest of the analysis. We can further derive the number of samples required as in the sub-exponential case, which will turn out to be the same for the given values of $\mu$ and $\delta$.
In Algorithm \ref{alg:B-MOSS}, we will now choose $\alpha = T^{\frac{ - 2\hat{\beta}}{1+2\hat{\beta}}}$ instead of $T^{\frac{ - 2{\beta}}{1+2{\beta}}}$ (which is the regret optimal $\alpha$, from Theorem~\ref{thm:sublinearOne}). We now have the following theorem on regret.


\begin{theorem}
Let $\hat{\beta}$ be a learned distributional parameter such that it satisfies Inequality~(\ref{eq:prob2}). If \Alg (Algorithm \ref{alg:B-MOSS}) is run with $\hat{\beta}$, then with probability at least $1-\delta$, the regret under  \Subp tail distribution assumption is upper bounded by $ O\left(T^{1-\frac{\beta(1-\mu\beta+\mu)}{ 1+2\beta-3\mu\beta+3\mu} }\right)$ if $\hat{\beta} > \beta$ and by $ O\left(T^{\frac{1+\beta + 2\mu(\beta -1)}{1+2\beta+3\mu(\beta-1)}} \right)$ if $\hat{\beta} \le \beta$. In both the cases, the regret is sub-linear.
\label{thm:uncertainty_subpareto}
\end{theorem}


\begin{proof}
For $c = 1/2$
(regret optimal value in Theorem \ref{thm:sublinearOne}), the following is true with   probability at least $1-\delta$.

\begin{align*}
& \alpha \ge T^{\frac{-\hat{\beta}}{\hat{\beta}+1/2}}
\\
\implies &   \alpha^{1-\frac{\hat{\beta}}{\hat{\beta}+1/2}} \ge  (\alpha T)^{\frac{-\hat{\beta}}{\hat{\beta}+1/2}}
\\
\implies &   \alpha^{\frac{1/2}{\hat{\beta}+1/2}} \ge  (\alpha T)^{\frac{-\hat{\beta}}{\hat{\beta}+1/2}}
\\
\implies &   \alpha^{\frac{\beta}{2\hat{\beta}}} \ge  (\alpha T)^{-\beta}
\\
    \implies & F_X(\alpha T) > 1-(\alpha T)^{-\beta} \ge 1-\alpha^{\frac{\beta}{2\hat{\beta}}}
    \tag*{(sub-Pareto tail assumption)}
\end{align*}
Thus, by Lemma~\ref{lem:regret}, we have $R_{\Alg}(T) \le T(1-(1-6\sqrt{\alpha})F_X(\alpha T)) < T\left(1-(1-6\sqrt{\alpha})(1-\alpha^{\frac{\beta}{2\hat{\beta}}})\right)$. Hence,
\begin{align*}
    R_{\Alg} & < T\Big(6\sqrt{\alpha} + \alpha^{\frac{\beta}{2\hat{\beta}}} - 6\alpha^{\frac{\beta}{2\hat{\beta}} + \frac{1}{2}}\Big)
    < T\Big(6\sqrt{\alpha} + \alpha^{\frac{\beta}{2\hat{\beta}}}\Big)
\end{align*}

\noindent

\noindent\textbf{Case 1} ($\hat{\beta}> \beta $):
Note that, $\frac{\beta}{\beta-1}-\mu \le \frac{\hat{\beta}}{\hat{\beta}-1}$  is equivalent to
\begin{equation}
    \hat{\beta} \le  \frac{\beta - \mu\beta +\mu}{1-\mu\beta+\mu}
    \label{eqn:betahat_upper_bound}
\end{equation}
We further have $\frac{\hat{\beta}}{\hat{\beta} - 1} > 1$ and hence, the lower bound estimate on the mean should also be greater than $1$. The lower bound estimate of $\frac{\hat{\beta}}{\hat{\beta}-1}$ with probability $1-\delta$ is: $ \frac{\beta}{\beta-1} - \mu$. Thus, $\frac{\beta}{\beta-1} - \mu> 1$, which  gives us that for the case $\hat{\beta}>\beta$:
\begin{equation}
    \mu < \frac{1}{\beta-1}
    \label{eqn:mu_bound_pareto}
\end{equation}
For this case,
$\alpha^{\frac{\beta}{2\hat{\beta}}} > \sqrt{\alpha} \;(\because \alpha \leq 1)$, and hence,
\begin{align*}
    R_{\Alg} & \leq O\Big(T \alpha^{\frac{\beta}{2\hat{\beta}}}\Big)
    \\
     & = O\Bigg(T^{1-\frac{\beta}{2\hat{\beta}} \big( \frac{2\hat{\beta}}{2\hat{\beta} +1} \big)}\Bigg) \tag*{($\because  \alpha = T^{\frac{-2\hat{\beta}}{2\hat{\beta} +1}}$ minimizes regret (Theorem \ref{thm:sublinearOne}))}
    \\
    & = O\Bigg(T^{1-\frac{\beta}{2\hat{\beta} +1} }\Bigg)
    \\
    & \leq O\left(T^{1-\frac{\beta}{2\left( \frac{\beta - \mu\beta +\mu}{1-\mu\beta+\mu} \right) +1} }\right)
     \tag*{($\because$ the upper bound on $\hat{\beta}$ is $\frac{\beta - \mu\beta +\mu}{1-\mu\beta+\mu}$  w.p. at least $1-\delta$ (Inequality~(\ref{eqn:betahat_upper_bound})))}
     \\
     & = O\left(T^{1-\frac{\beta(1-\mu\beta+\mu)}{ 1+2\beta-3\mu\beta+3\mu} }\right)
\end{align*}
Note that the regret is sub-linear in $T$ since $\frac{\beta(1-\mu\beta+\mu)}{ 1+2\beta-3\mu\beta+3\mu } > 0$ ($\because  \mu < \frac{1}{\beta-1}$ for this case, from Inequality~(\ref{eqn:mu_bound_pareto})).
\\

\noindent
\textbf{Case 2} ($\hat{\beta}\leq \beta $):
We have $\frac{\hat{\beta}}{\hat{\beta}-1} \leq \frac{{\beta}}{{\beta}-1} + \mu$ with probability at least $1-\delta$, which is equivalent to
\begin{equation}
\hat{\beta} \geq \frac{\beta + \mu\beta -\mu}{1+\mu\beta-\mu}
    \label{eqn:betahat_lower_bound}
\end{equation}

For this case,
$\alpha^{\frac{\beta}{2\hat{\beta}}} \leq \sqrt{\alpha} \;(\because \alpha \leq 1)$, and hence,

\begin{align*}
    R_{\Alg} & \leq O\Big(T\sqrt{\alpha}\Big)
    \\
     & = O\Bigg(T^{1-  \frac{1}{2 +\frac{1}{\hat{\beta}}} }\Bigg) \tag*{($\because  \alpha = T^{\frac{-2\hat{\beta}}{2\hat{\beta} +1}}  = T^{\frac{-2}{2 +\frac{1}{\hat{\beta}}}}$ (Theorem \ref{thm:sublinearOne}))}
     \\
      & \leq O\Bigg(T^{1-  \frac{1}{2 +\frac{1+\mu\beta-\mu}{\beta + \mu\beta -\mu}} }\Bigg)
      \tag*{($\because$ the lower bound on $\hat{\beta} $ is $\frac{\beta + \mu\beta -\mu}{1+\mu\beta-\mu} $ w.p. at least $1-\delta$ (Inequality~(\ref{eqn:betahat_lower_bound})))}
      \\
      & = O\left(T^{1- \frac{\beta + \mu\beta -\mu}{1+2\beta+3\mu\beta-3\mu}} \right)
      \\
      & = O\left(T^{\frac{1+\beta + 2\mu(\beta -1)}{1+2\beta+3\mu(\beta-1)}} \right)
\end{align*}

Note that the regret is sub-linear in $T$.
\\

Note also that in both the cases, when $\mu$ tends to zero, the regret tends to go to the original regret of $ O(T^{\frac{1+\beta}{1+2\beta}})$ (Theorem \ref{thm:sublinearOne}), with probability at least $1-\delta$.






\end{proof}

~\\

\noindent\textbf{Cost of Parametric Uncertainty (CPU):}
We now quantify the robustness of the regret guarantee provided by BL-Moss towards uncertainty in the tail distribution of the best arm's arrival.
We define CPU to be the ratio of the regret achieved with the learned parameters and the regret achieved with the actual parameters. From the above theorems, we have the following:

\begin{enumerate}[leftmargin=*]
    \item CPU for sub-exponential tail distribution is given as
    (on absorbing the constant parameters $\hat{\lambda}$, $\lambda$, and $\mu$ in order notation):

    \begin{align*}
    \allowdisplaybreaks
        CPU(\lambda, \mu)
        &=\begin{cases}   \frac{T^{\frac{1+\mu\lambda }{2}} \Big(\sqrt{\log(T)}\Big)^{1-\mu\lambda}}{\sqrt{T\log(T)}}
        = \left( \sqrt{\frac{T}{\log(T)}} \right)^{\mu\lambda} & \text{if } \mu\lambda < 1 \\[1em] \frac{\sqrt{T\log(T)}}{\sqrt{T\log(T)}} = 1 & \text{if } \mu\lambda \ge 1 \end{cases}
    \end{align*}

    \item CPU for \Subp tail distribution is given as:
    \begin{align*}
        CPU(\beta,\mu) &=\begin{cases}
         \frac{T^{1-\frac{\beta(1-\mu\beta+\mu)}{ 1+2\beta-3\mu\beta+3\mu } }}{T^{\frac{1+\beta}{ 1+2\beta} }} = T^{\frac{2\mu\beta(\beta-1)}{(1+2\beta)(1+2\beta-3\mu(\beta-1))}} & \text{if } \hat{\beta} > \beta \\[1em]
          \frac{T^{\frac{1+\beta + 2\mu(\beta -1)}{1+2\beta+3\mu(\beta-1)}}}{T^{\frac{1+\beta}{ 1+2\beta} }} = T^{\frac{\mu(\beta-1)^2}{(1+2\beta)(1+2\beta+3\mu(\beta-1))}} & \text{if } \hat{\beta} \leq \beta
        \end{cases}
    \end{align*}

\end{enumerate}

\subsection{Arm Arrival Distribution}
\label{ssec:arrivalassumptions}

Throughout the paper, we considered a setting that assumed certain distributions on the arrival of the best arm, namely, the best arm is more likely to arrive in early rounds.   In this section, we  discuss another practically relevant setting, which assumes distribution on the arrival rate of arms with time. If arms arrive at a faster rate in early rounds, that is, if a large fraction of arms arrive relatively early,  it can be shown that one can use our proposed \Alg algorithm.  The following result establishes the equivalence between the distributional assumptions in the two settings.
\begin{theorem}
Let $f(t)$ denote the fraction of arms arrived till time $t$. Further, let the quality of each arriving  arm be an i.i.d. sample  from \textsf{unif[0, 1]}. Then, the best arm's arrival distribution $F_X(t)$ satisfies  $F_X(t)  = f(t)$.
\label{thm:equivalence}
\end{theorem}
\begin{proof}
Let $M \geq1$ be the total number of arms arrived till time $T$. First, observe that $\mathbbm{P}(i = i^{\star}) = \frac{1}{M}$. Here, $i^{\star} = \arg\max_{i\in [M]} q_i$. We have
\begin{align*}
F_X(t) &= \sum_{\ell=1}^{t} M \cdot (f(\ell) - f(\ell -1)) \mathbbm{P}(i_{\ell} = i^{\star}) \\ & = M\cdot (f(t) - f(0)) \frac{1}{M} = f(t) \tag*{($\because f(0) = 0$)}
\end{align*}
\end{proof}
A fundamental difference between the two settings is that, in the first setting, we consider that  a new arm  arrives at each time instant; whereas in the second setting, we consider that arms follow an arrival process having a sharp tail. In the first setting, our proposed algorithm achieves sublinear regret due to early arrival of the best arm.  In the second setting, even if  each arm is equally likely to be the best arm, our algorithm achieves sublinear regret owing to most of the arms arriving relatively early.


\begin{figure}[t!]
\centering
	\includegraphics[width= \linewidth]{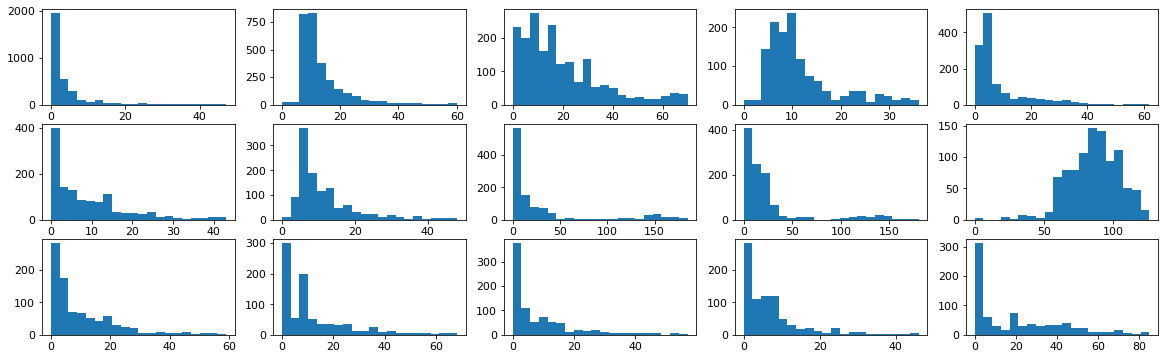}
	\caption{Arrival distribution of reviews for representative Digital Music products on Amazon (X-axis: number of months elapsed since the first review, Y-axis: number of reviews)}
	\label{fig:digitalmusic}
\end{figure}

We now validate our  distributional assumption on the arrival of arms, with real-world data such as posting times of answers for questions on StackExchange\footnote{StackExchange data dump is available publicly at \cite{SE20}.} and posting times of reviews for products on Amazon\footnote{Amazon review data is available publicly at \cite{NI18}.} and Steam\footnote{Steam video game and bundle data is available publicly at \cite{ST17}.}.  Figure \ref{fig:digitalmusic} presents the arrival distribution of reviews over time for a representative set of the most popular digital music products on Amazon.   Each subfigure shows the arrival of reviews for a particular product. In each subfigure, the X-axis represents the number of months elapsed since the first posted review for that product, and the Y-axis represents the number of reviews.
It is to be noted that in  MAB  applications, X-axis usually represents the number of opportunities to pull the arms (here, the cumulative number of views to the reviews on a given product page) and not the wall-clock time (here, the number of months). We assume that the number of such views does not change significantly across different time intervals, and hence consider the wall-clock time as a proxy for the number of opportunities for arm pulls.

\begin{figure}[t!]
\centering
	\includegraphics[width= \linewidth]{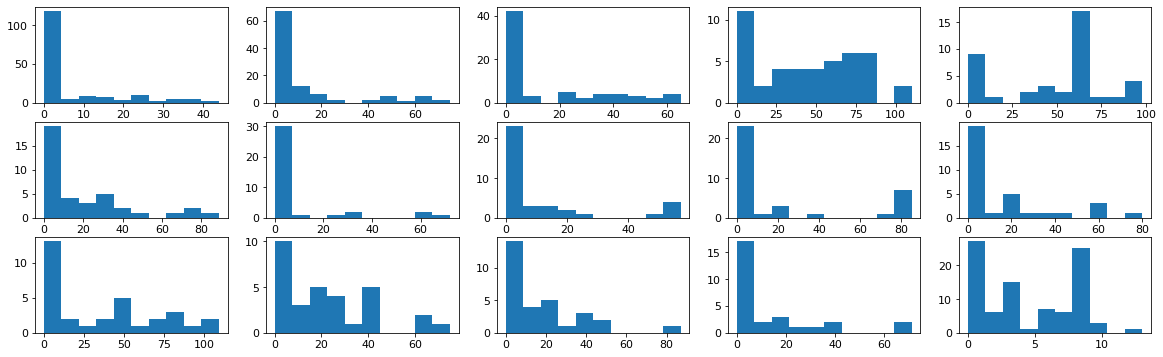}
	\caption{Arrival distribution of answers for representative questions on Mathematics StackExchange (X-axis: number of months since the first posted answer, Y-axis: number of answers)}
	\label{fig:mathexchange}
\end{figure}

We observe that for most products, the number of reviews  follows a decreasing trend over time (i.e., a large fraction of the reviews arrive early), which is aptly captured by sub-exponential and \Subp tail distributions.
A very similar trend was observed for questions on various sub-domains of StackExchange, reviews on Amazon products belonging to other categories like CDs \& vinyls, video games, software, movies and TV, etc., as well as reviews on video games on Steam.
Figure \ref{fig:mathexchange} presents the arrival distribution of answers over time for a representative set of the most answered questions on the Mathematics StackExchange platform.
Likewise, the arrival distribution of reviews for representative products belonging to categories of CDs \& Vinyls and video games are respectively presented in Figures \ref{fig:CDsVinyl} and \ref{fig:videogames}  in  Appendix \ref{app:simulation}.


It is relevant to note here that this trend is not necessarily followed across all product categories.
For instance, certain products would follow upticks or a gradual increase in the number of posted reviews, either due to marketing campaigns (through media advertising or word-of-mouth publicity) or spike in demand during festive seasons.
Amazon gift cards, cell phones and accessories, etc. are popular examples of such products (Figure \ref{fig:giftCards} in Appendix \ref{app:simulation} presents the distribution for Amazon gift cards).

~\\
{\em A note to practitioners\/}:
 In practice, the arrival distribution of reviews for  any product would depend on the type of the product, marketing strategy of the product manufacturer,  true quality of the product, and so on. Though the proposed \Bmab\ framework provides curation strategy to identify the best arm (user generated content)  from ever increasing choices, one must use expert knowledge about arrival rate of the arms, qualities of arriving arms, total expected number of arms, and so on, so as to design an optimal learning \Bmab\ algorithm. We leave mathematical modeling and design of specialized \Bmab\ algorithms which take into account the specific arrival of arms, as an interesting future direction to our work.

 \subsection*{A Remark on Unknown Distributional Parameters of Arm Arrival Distribution}


 Note that if we have uncertainty with respect to the distributional parameters (discussed in Section~\ref{ssec:unknownparam})
 in the setting which makes distributional assumption on the rate of arrival of arms (discussed in Section~\ref{ssec:arrivalassumptions}), we could transform it into the setting which makes distributional assumption on the arrival time of the best arm using Theorem~\ref{thm:equivalence}, and then show that the sub-linearity of regret is preserved even if we have uncertainty with respect to the distributional parameters (using Theorem~\ref{thm:uncertainty_subexp} or \ref{thm:uncertainty_subpareto}). Thus, though the distributional parameters signify very different things in the two settings, one can prove that our algorithm would achieve sub-linear regret in such a combined case.

\section{Additional Related Work}
\label{sec:BL-related}
A standard stochastic MAB framework considers that the number of available arms is fixed (say $k$) and typically much less than the time horizon (say $T$). In the seminal work of Lai and Robbins~\cite{lai85}, the authors showed that any MAB algorithm in such a setting must incur a regret of $\Omega(\frac{\log T}{D_{\text{KL}}})$ where $D_{\text{KL}}$ is the Kullback-Leibler
divergence between the best arm and the second best arm. Auer~\cite{auer2002using} proposed the UCB1 algorithm which attains a matching upper bound on the expected regret. However, the distribution-free (i.e., in adversarial case) regret of the variant of UCB1, $(\alpha, \psi)$-UCB,  is given by  $ O( \sqrt{kT\log T})$~\cite{bubeck2012regret}. The \Moss algorithm proposed by Audibert and Bubeck~\cite{audibert2010regret} achieves the distribution-free regret of $O(\sqrt{kT})$. 
Bubeck and Cesa-Bianchi~\cite{bubeck2012regret} present
 a detailed survey on regret bounds of these algorithms.
 
 A setting similar to ballooning bandits is studied under Markovian bandits framework; where each arm is characterized  by a known MDP. This setting, known as  \emph{arm-acquiring bandits}~\cite{whittle1981arm}  was first studied by~\cite{nash1973optimal}. In arm acquiring bandits framework the goal is to maximize the discounted, infinite time cumulative reward whereas in ballooning bandits goal is to minimize the finite time cumulative regret. The difference in the two models is further highlighted by the fact that ballooning bandits is a \emph{learning} problem whereas arm-acquiring bandits is a planning problem.

The problem of learning qualities of the answers on Q\&A forums was first modeled under MAB framework by Ghosh and Hummel~\cite{ghosh2013learning} where generation of a new arm was considered as a consequence of strategic choice of an agent. Though this model captures strategic aspects of the contributors, there is an important practical issue with such modelling. Each agent, being a strategic attention seeker, is assumed to produce the effort just enough to satisfy incentive compatibility in the equilibrium. We do not assume an efforts-and-costs model and show that, even when the number of answers grows linearly with time if the  qualities of arriving answers follow  certain mild distributional assumption,
the proposed algorithm achieves sub-linear regret. 
%

Tang and Ho~\cite{tang2019bandit} consider a model with fixed number of arms but with a platform where agents provide biased feedback. On such Q\&A forums, it is more relevant to consider the problem with increasing number of arms.
A recent work by Liu and Ho~\cite{LIU18} limits the growth of the bandit arms by randomly dropping some of the arms from consideration, and computing the regret with respect to only the considered arms. That is, they do not account for the regret incurred due to the randomly dropped arms.  



\section{Discussion and Future Work}
\label{sec:BL-discussion}

In this paper, we presented a novel extension to the classical MAB model, which we call the Ballooning bandits model (\bmab). We showed that,  it is impossible to attain a  sub-linear regret guarantee without any distributional assumption on the best arm's
arrival.
We proposed an algorithm for the \bmab\ model and provided sufficient conditions under which the proposed algorithm achieves sub-linear regret.  In particular, when the arrival distribution of the best quality arm has a sub-exponential or \Subp tail, our algorithm \Alg achieves sub-linear regret by restricting the number of arms to be explored in an intelligent way.  

Our results indicate that the number of arms to be explored  depends on the distributional parameters, namely, $\lambda$ (for sub-exponential case) and $\beta$ (for \Subp case), which must be known to the algorithm. However, in practice, these parameters may not be known exactly. We studied the increase in regret when one must use approximations of these values.  It will be interesting to see how a learning algorithm can be designed to learn these parameters. 
We also studied the effect of a varying rate of arrival of arms (instead of the arrival time of the best arm). Owing to our equivalence theorem, our algorithm and results are directly applicable to cases wherein arm arrivals follow a sub-exponential or \Subp structure. However, a general result with arbitrary (albeit sublinear) arrival of arms is still an open question.   One could also consider other  arrival processes for arms, in order to obtain tighter, arrival specific regret guarantees.

In this work, we employed \textsc{Moss} as the underlying learning algorithm owing to its simplicity and optimality,  in terms of both the number of arms and the time horizon. It is an interesting future direction to determine the threshold parameter $\alpha$ under other learning algorithms such as \textsc{Thompson Sampling}, \textsc{UCB1}, \textsc{KL-UCB}, and analyze the corresponding regret bounds.
We assumed the knowledge of time horizon, as is the case with several works on MAB. Note that even if the time horizon is not known, one could always work with its approximate value which is typically known from  past experiences. Extending our algorithm to the case of unknown time horizon using techniques such as  MOSS-anytime \cite{degenne2016anytime} or doubling trick \cite{besson2018doubling}, is a promising direction for future work.




\section*{Acknowledgement}
GG gratefully acknowledges   the support from MHRD fellowship, Govt. of India and  the Israeli Ministry of Science and Technology Grant 19400214.

\bibliographystyle{alpha}
\bibliography{mybibfile}
\appendix
\newpage
\appendix
\noindent {\huge \textbf{Appendices} }
\section{Omitted Proofs}
\label{sec:BL-appendix}
\begin{claim}
$\frac{W(\lambda T /c)}{\lambda T/c} < 1/36 \Longleftrightarrow T > \frac{36 c \log(36)}{\lambda } $
\label{clm:Tbound}
\end{claim}
\begin{proof}
We have the following equivalent inequalities.
\begin{align*}
   &\; \frac{W(\lambda T /c)}{\lambda T/c}  < \frac{1}{36}
    \\ \Longleftrightarrow &\; e^{-W(\lambda T /c)}  < \frac{1}{36} \;\;\;\;\;(\because W(x) e^{W(x)} = x)
    \\ \Longleftrightarrow &\; W(\lambda T /c) > \log(36)
     \\ \Longleftrightarrow &\; \frac{\lambda T}{c} > \log(36) e^{\log(36)}
     \\ \Longleftrightarrow &\; T > \frac{36 c \log(36)}{\lambda }
  \end{align*}
  The second to last inequality is obtained by applying the monotone increasing function $f(x):=x e^x$ on both sides, and then using Definition \ref{def:lambert} of Lambert $W$ function.
\end{proof}

\begin{claim}
\label{clm:decreasing}
$e^{-cW(\lambda T / c)}$ is decreasing in $c$ for $c \in (0,1/2]$.
\end{claim}
\begin{proof}

  For $c_1>c$, we have
  \begin{align*}
      &\; \lambda T/c  > \lambda T/c_1
      \\ \Longleftrightarrow &\; W(\lambda T/c)  > W(\lambda T/c_1) \;\;\;(\text{Property } \ref{prop:two} \text{ of Lambert $W$})
      \\ \Longleftrightarrow &\; e^{-W(\lambda T/c)}  < e^{-W(\lambda T/c_1)}
      \\ \Longleftrightarrow &\; \frac{W(\lambda T/c)}{\lambda T/c}  < \frac{W(\lambda T/c_1)}{\lambda T/c_1} \;\;\;(\because W(x) e^{W(x)} = x)
      \\ \Longleftrightarrow &\; c W(\lambda T/c)  < c_1 W(\lambda T/c_1)
      \\ \Longleftrightarrow &\; e^{- c W(\lambda T/c)}  > e^{- c_1 W(\lambda T/c_1)}
  \end{align*}
\end{proof}
\section{Properties of Lambert W function}
\label{app:prop}
\PropZero*
\begin{proof}
	The forward direction is straightforward. As $W(\cdot)$ is one to one function in the non-negative domain,  we have $W(W(x)e^{W(x)}) =  W(x)$. Let $y = W(x)$ then we have $W(ye^{y}) = y$. To show that $W(xe^{x}) = x$ implies $W(z)e^{W(z)} = z$, observe that $W(W(z_0)e^{W(z_0)}) = W(z_0)$.  We get the required result by taking the inverse.
\end{proof}

\PropOne*
\begin{proof}
	By definition, the Lambert $W$ function satisfies $W(x)e^{W(x)} = x$. It is easy to see that $W(e) = 1$. Further we have,
	\begin{align*}
	1 & = \frac{dW(x)}{dx} \cdot e^{W(x)} + \frac{dW(x)}{dx} \cdot e^{W(x)} W(x)  \\
	& = \frac{dW(x)}{dx} (x + e^{W(x)}) \\
	\implies  \frac{dW(x)}{dx} &= \frac{1}{x + e^{W(x)}}
	\end{align*}
	Let $f(x) = \log(x) - W(x)$. We have that $f(e) = 0$. We have that $\frac{df(x)}{dx} = \frac{1}{x} - \frac{1}{x + e^{W(x)}} = \frac{e^{W(x)}}{x(x + e^{W(x)})} = \frac{1}{x(1+W(x))} > 0$. Hence we have that $f(\cdot)$ is increasing i.e. $f(x) >0$ for all $x > e$. This shows that $W(x) \leq \log(x)$.

	Now, let $g(x)=\frac{\log(x)}{2}-W(x)$. Here, we have $g(e)<0$, and $\frac{dg(x)}{dx} = \frac{1}{2x}- \frac{1}{x + e^{W(x)}} = \frac{e^{W(x)}-e^{\log(x)}}{2x(x + e^{W(x)})} \leq 0$ (since $W(x)\leq \log(x)$ for $x\geq e$).
	So, $g(x)<0$ for all $x \geq e$, implying that $\frac{\log(x)}{2}<W(x)$.
	This completes the proof.
\end{proof}
\PropTwo*

\begin{proof}
	Observe that $W(0) = 0$. Note that in the non-negative domain, $ f(x) =xe^{x}$ is continuous, one to one and strictly increasing. Hence, its inverse, $W(\cdot)$,  is also increasing.
\end{proof}

\newpage
\section{Validating Arm Arrival Distributions}
\label{app:simulation}


In each subfigure, X-axis represents the number of months elapsed since the first posted review for the corresponding product, and Y-axis represents the number of reviews.

\begin{figure}[htb]
\centering
	\includegraphics[width= \linewidth]{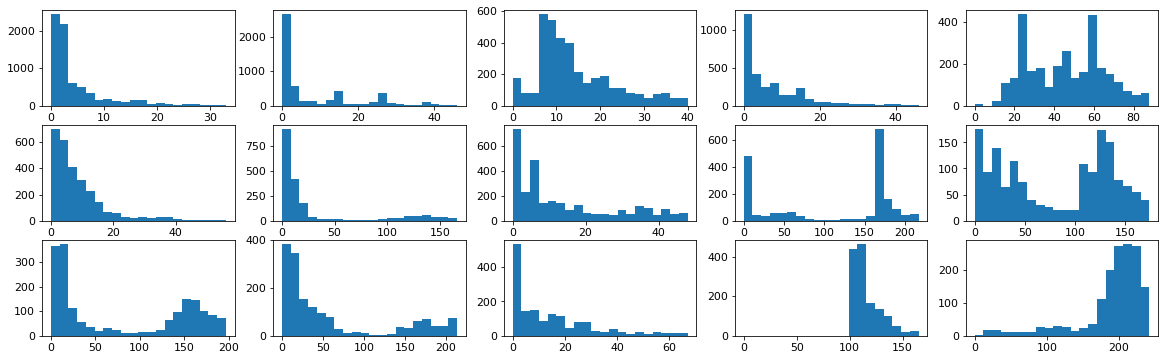}
	\caption{Arrival distribution of reviews for representative CDs \& vinyl products on Amazon}
\label{fig:CDsVinyl}
\end{figure}

\begin{figure}[htb]
\centering
	\includegraphics[width= \linewidth]{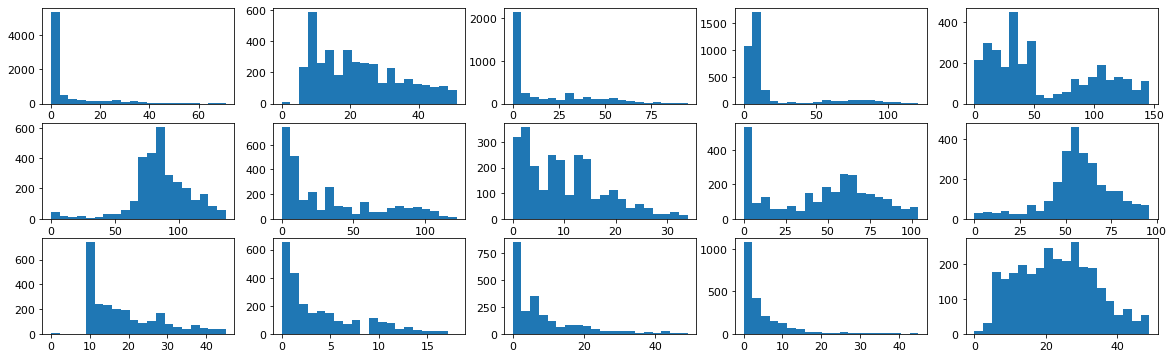}
	\caption{Arrival distribution of reviews for representative video game products on Amazon}
	\label{fig:videogames}
\end{figure}
\begin{figure}[htb]
\centering
	\includegraphics[width= \linewidth]{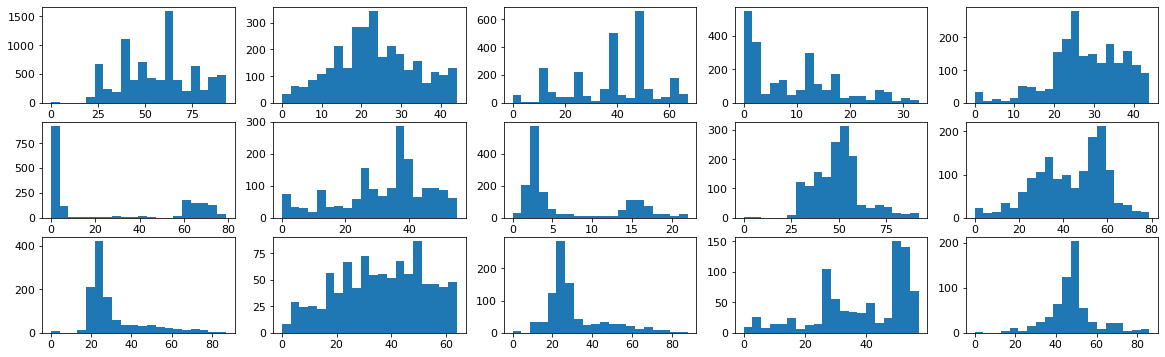}
	\caption{Arrival distribution of reviews for representative gift card products on Amazon}
\label{fig:giftCards}
	\vspace{-1mm}
\end{figure}

\end{document}